\pdfoutput=1
\documentclass[twoside,11pt]{article}
\usepackage{jmlr2e}
\usepackage{hyperref}
 \usepackage{bbm}
\usepackage{lastpage}  % to get total number of pages in initial submission; remove for final
\usepackage{algorithmic}
\usepackage{algorithm}
\usepackage{amsbsy}
\usepackage{amsfonts}
\usepackage{amsmath}
\usepackage{isomath}
\usepackage{xspace}
\usepackage{tikz}
\usepackage{gnuplot-lua-tikz}
\usepackage{pgfplots}
\usepackage{varioref}
\usepackage{subcaption}

\usepackage{dsfont} %% font for Real numbers domain
\usepackage{enumerate} %% options to enumerate
\usepackage{mathrsfs} %% math fonts
\usepackage{paralist} %% enumeration
\usepackage{url} %% use \url

\newtheorem{assumption}{Assumption}
\newtheorem{claim}{Claim}

\newcommand \E {\mathop{\mbox{\ensuremath{\mathbb{E}}}}\nolimits}

\renewcommand \Pr {\mathop{\mbox{\ensuremath{\mathbb{P}}}}\nolimits}

\newcommand{\set}[1]{\left\{\, #1 \,\right\} }
\newcommand{\cset}[2]{\left\{\, #1 \mathrel{:} #2 \,\right\} }

\newcommand{\stat} {\tau}

\newcommand\Reals {{\mathds{R}}}
\newcommand\Naturals {{\mathds{N}}}

\newcommand{\data}[0]{\ensuremath{D}\xspace}

\newcommand{\eg}[0]{\emph{e.g.},\xspace}
\newcommand{\ie}[0]{\emph{i.e.},\xspace}

\newcommand \CV {{\mathcal{V}}}
\newcommand \CS {{\mathcal{S}}}
\newcommand \CX {{\mathcal{X}}}

\newcommand \Bay {\ensuremath{\mathscr{B}}}
\newcommand \Adv {\ensuremath{\mathscr{A}}}

\newcommand \Params {\Theta}
\newcommand \param {\theta}
\newcommand \mlparam {\theta^\star_{\textrm{ML}}}

\newcommand \Utils{\mathcal{U}}
\newcommand \util{u}
\newcommand \Queries {\mathcal{Q}}
\newcommand \query {q}
\newcommand \Responses {\mathcal{R}}
\newcommand \response {r}
\newcommand \family {\mathcal{F}_{\Params}}
\newcommand \bel {\xi}
\newcommand \Bels {\Xi}

\newcommand \marg {\phi}
\newcommand \dom {\nu}

\newcommand \defn {\mathrel{\triangleq}}

\newcommand \argmax{\mathop{\rm arg\,max}}

\newcommand \degree[1]{\mathop{\rm deg}(#1)}

\newcommand \abs[1] {\left|#1\right|}

\newcommand \onenorm[1]{\left\|#1\right\|_1}

\newcommand \norm[1]{\left\|#1\right\|}

\DeclareMathAlphabet{\mathpzc}{OT1}{pzc}{m}{it}

\newcommand \Beta {\mathop{\mathpzc{Beta}}\nolimits}
\newcommand \Binomial {\mathop{\mathpzc{Binom}}\nolimits}

\newcommand \Exp{\mathop{\mathpzc{Exp}}\nolimits}

\newcommand\ind[1]{\mathop{\mbox{\ensuremath{\mathbb{I}}}}\left\{#1\right\}}

\newcommand\tsum{\textstyle\sum}

\newcommand \BigO[1]{O\left(#1\right)}

\newcommand \KL[2] {D\left(#1\middle\|#2\right)}
\newcommand \partition {\mathscr{M}}
\newcommand \algebra {\mathfrak{S}}
\newcommand \field[1] {\algebra_{#1}}
\newcommand \nsamples {N}

\newcommand{\Parents}[1]{\mathcal{P}(#1)}

\newcommand \xdistChar{\rho}
\newcommand \xdist[2]{\xdistChar(#1, #2)}
\newcommand \dist[2]{D\left(#1 ~\middle\|~ #2 \right)}
\newcommand\dd{\,\mathrm{d}}

\def\clap#1{\hbox to 0pt{\hss#1\hss}}

\def\mathrlap{\mathpalette\mathrlapinternal}

\def\mathrlapinternal#1#2{%
           \rlap{$\mathsurround=0pt#1{#2}$}}

\newcommand \comment[1] {{\color{red}{\texttt{[#1]}}}}

\newcommand{\Ie}{\emph{I.e.,}\xspace}

\newcommand \constScale {\omega}
\newcommand \constFamily {C_\bel^{\family}}

\newenvironment{proofof}[1]{\par\noindent{\bf Proof of {#1}\ }}{\hfill\BlackBox\\[2mm]}

\jmlrheading{16}{2016}{1--\pageref{LastPage}}{5/15; Revised 12/16}{}{C. Dimitrakakis, B. Nelson, Z. Zhang, A. Mitrokotsa, B. I. P. Rubinstein}

\ShortHeadings{Bayesian Differential Privacy through Posterior Sampling}{Dimitrakakis, Nelson, Zhang, Mitrokotsa and Rubinstein}
\firstpageno{1}

\begin{document}
% first the title is needed
\title{Bayesian Differential Privacy through Posterior Sampling\thanks{A preliminary version of this paper appeared in {\em Algorithmic Learning Theory 2014}~\citep{alt:robust}. This version corrects constant factors in the upper bounds and introduces new material on utility analysis and lower bounds.}}

\author{\name Christos Dimitrakakis \email christos.dimitrakakis@gmail.com \\
	\addr
    University of Lille, F-59650 Villeneuve-d'Ascq, France\\
    Harvard University, Cambridge MA-02138, USA\\
	Chalmers University of Technology, SE-412 96, Gothenburg, Sweden
  \AND  \name Blaine Nelson \email blaine.nelson@google.com \\
        \addr Google, Inc.\\
	1600 Amphitheatre Parkway \\
	Mountain View, CA 94043, USA
  \AND  \name Zuhe Zhang \email zhang.zuhe@gmail.com \\
        \addr School of Mathematics \& Statistics \\
        The University of Melbourne \\
        Parkville, VIC 3010, Australia
  \AND  \name Aikaterini Mitrokotsa \email aikmitr@chalmers.se \\
        \addr Department of Computer Science \& Engineering \\
	Chalmers University of Technology \\
	SE-412 96, Gothenburg, Sweden
  \AND  \name Benjamin I. P. Rubinstein \email brubinstein@unimelb.edu.au \\
        \addr Department of Computing \& Information Systems \\
        The University of Melbourne \\
	Parkville, VIC 3010, Australia
}
\editor{Charles Elkan}

\maketitle

\begin{abstract}% (adding this seems to remove the indent --> good)
  Differential privacy formalises privacy-preserving mechanisms that
  provide access to a database. We pose the question of whether
  Bayesian inference itself can be used directly to provide private
  access to data, with no modification.  The answer is affirmative:
  under certain conditions on the prior, sampling from the posterior
  distribution can be used to achieve a desired level of privacy and
  utility. To do so, we generalise differential privacy to arbitrary
  dataset metrics, outcome spaces and distribution families.  This
  allows us to also deal with non-i.i.d or non-tabular datasets. We
  prove bounds on the sensitivity of the posterior to the data, which
  gives a measure of robustness. We also show how to use posterior
  sampling to provide differentially private responses to queries,
  within a decision-theoretic framework.  Finally, we provide bounds
  on the utility and on the distinguishability of datasets. The latter
  are complemented by a novel use of Le Cam's method to obtain lower
  bounds.  All our general results hold for arbitrary database
  metrics, including those for the common definition of differential
  privacy. For specific choices of the metric, we give a number
  of examples satisfying our assumptions.
\end{abstract}


\section{Introduction}

The Bayesian framework for statistical decision theory incorporates uncertainty into decision making in a probabilistic manner. This makes it attractive, as predictions and modelling can all be made with the machinery of probability. More specifically, a Bayesian statistician begins by assuming that the world is described by a probabilistic model within some family, and he assigns a prior belief to each one of the models. After observing data, this belief is adjusted through Bayes's theorem to the so called posterior belief. This expresses the statistician's conclusion given the data and the prior assumptions. The statistician can then release the posterior to the world, for others to build upon, or use for principled decision making under uncertainty.

Unfortunately, it is frequently the case that the data acquired by the statistician is sensitive. Consequently, there is a fear that any information released by the statistician that depends on the data---be that the posterior distribution itself or any decisions that follow from the calculated posterior---may reveal sensitive information in the original data. Recently, the framework of differential privacy has been proposed to codify this leakage of information. If an algorithm is differentially private, then its output can only leak a bounded amount of information about its input.

We are interested in the question of how we can build differentially-private algorithms within the Bayesian framework. More precisely, we examine when the choice of prior is sufficient to guarantee differential privacy for decisions that are derived from the posterior distribution. Our work builds a unified understanding of privacy and learning in adversarial environments, under a decision-theoretic framework. We show that under suitable assumptions, standard Bayesian inference and posterior sampling can achieve uniformly good utility with a fixed privacy budget in the differential privacy setting. We also indicate strong connections between robustness and privacy.

In this paper, we show that the Bayesian statistician's choice of
prior distribution ensures a base level of data privacy through the
posterior distribution; the statistician can safely respond to
external queries using samples from the posterior. When estimating a
linear model from sensitive data, for example, samples from the
posterior correspond to different possible fits. The more samples
used, the more privacy is leaked, while query responses may be more
accurate. Our proposed approach complements existing mechanisms rather
well, and may be particularly useful in situations where Bayesian
inference is already in use. For this reason, we provide illustrative
examples in the exponential family. However, our setting is wholly
general and not limited to specific distribution families, or
i.i.d. observations. Any family could be chosen: so long as it either
satisfies our assumptions directly, or can be restricted so that it
does.  For example, our framework applies to families of discrete Bayesian
networks with directed-acyclic topologies (\eg Markov chains; see Lemma~\vref{lem:dbn})
and multivariate Gaussians (see Lemma~\ref{lem:multinormal}), where
the observations may not satisfy the i.i.d. assumption.

\emph{Summary of setting.} A Bayesian statistician (\Bay) wishes to communicate results about data $x$ to a third party (\Adv), but without revealing the data $x$ itself. We make no assumptions on the  data $x$, which could be a single observation, an i.i.d. sample, or a sequence of observations. The protocol of interaction between \Bay{} and \Adv{} is summarised below.
\begin{enumerate}
\item \Bay{} selects a model family ($\family$) and a prior ($\bel$).
\item \Adv{} is allowed to see $\family$ and $\bel$ and is computationally unbounded.
\item \Bay{} observes data $x$ and calculates the posterior $\bel(\param \mid x)$ but does not reveal it.\\ % Instead, \Bay{} draws $n$ samples of the parameter $\param$ from the posterior distribution, forming a sample $\hat{\Params}$.
Then, for steps $t = 1, 2, \ldots$, repeat the following:
\item \Adv{} sends his utility function $\util$ and a query $\query_t$ to \Bay{}.
\item \Bay{} responds with the response $\response_t$ maximising $\util$, in a manner that depends on the query and the posterior.
\end{enumerate}
Let us now elaborate. In this framework, the choice of the model family $\family$ is
dictated by the problem. The choice of $\bel$ is normally determined
by the prior knowledge of \Bay{}, but we show that this also affects
what level of privacy is achieved. Informally speaking, informative
priors achieve better privacy, as the posterior has a weaker
dependency on the data. It is natural to assume that the prior itself
is public, as it should reflect publicly available information. The
statistician's conclusion from the observed data $x$ is then
summarised in the posterior distribution $\bel(\param \mid x)$, which
remains private.

The second part of the process is the interaction with \Adv{}. We
adopt a decision-theoretic viewpoint to characterise what the optimal
responses to queries should be.  More specifically, we assume the
existence of a ``true'' parameter $\param \in \Params$, and that \Adv{}
has a utility function $\util_\param(\query_t, \response_t)$, which he
wishes to maximise. For example,
consider the case where $\param = (\mu, \Sigma)$ are the parameters of a normal
distribution. An example query $\query_t$ is \emph{``what is the expected value $\E_\param x_i = \mu$ of the distribution?''}. The optimal response $\response_t$, would then be a real
vector that depends on the utility function. A possible utility function is the negative squared
$L_2$ distance:
\[
u_\theta(\query_t = \textrm{``what is the mean?"}, \response_t) = - \|\E_\theta x_i - \response_t\|_2^2.
\]
While $\param$ is unknown, \Bay{} has information about it in the form
of a posterior distribution. Using standard
decision-theoretic
notions, the optimal response of $\Bay$ would maximise the expected
utility $\E_\bel(\util \mid \query_t, \response_t, x)$, where the
expectation is taken over the posterior distribution. However, this deterministic response cannot be differentially private.

In this paper, we promote the use of \emph{posterior sampling} to
respond to queries.
The posterior sampling mechanism draws a set $\hat{\Params}$ of
i.i.d. samples from the posterior distribution. Then, all the responses
only depend on the posterior through $\hat{\Params}$.  Since our
algorithm only takes a single sample set $\hat{\Params}$, further
queries by the adversary reveal nothing more about the data than what can
be inferred from $\hat{\Params}$. The empirical distribution induced by 
$\hat{\Params}$ serves as a private surrogate for the exact (non-private)
posterior. Consequently, we can respond to an
arbitrary number of queries with a bounded privacy budget, while
guaranteeing good utility for all responses.

We show that if $\family$ and $\bel$ are chosen
appropriately, this results in differentially-private responses, as
well as robustness of the posterior.\footnote{More specifically, that
  small changes in the data result in small changes in the posterior
  in terms of the KL divergence.}  In addition, we prove upper and
lower bounds on how easy it is for an adversary to distinguish two
$\epsilon$-close datasets. Finally, we bound the loss in utility
incurred due to privacy.
The intuition behind our results is that robustness and privacy are linked via smoothness. Learning algorithms that are smooth mappings---their output (\eg a spam filter) varies little with perturbations to input (\eg similar training corpora)---are robust: outliers have reduced influence, and adversaries cannot easily discover unknown information about the data. This suggests that robustness and privacy can be simultaneously achieved and are in fact deeply linked.

We provide a uniform mathematical treatment of the privacy and
robustness properties of Bayesian inference based on generalised
differential privacy to arbitrary dataset distances, outcome spaces,
and distribution families.  This paper can be summarised as making the following
distinct contributions:
\begin{itemize}
\item Under certain regularity conditions on the prior distribution $\bel$ or likelihood family $\family$,
we show that the posterior distribution is \emph{robust}: small changes in the dataset result in small posterior changes.
\item We introduce a novel \emph{posterior sampling mechanism} that is private.\footnote{Although previously used \eg for efficient exploration in reinforcement learning~\citep{thompson1933lou,osband:thompson:nips:2013}, posterior sampling has not previously been employed for privacy.}  Unlike other common mechanisms in differential privacy,
our approach sits squarely in the non-private (Bayesian) learning framework without modification.
\item We provide necessary and sufficient conditions for differentially private Bayesian inference.
\item We introduce the notion of \emph{dataset distinguishability} for which we provide finite-sample bounds for our mechanism: how large would $\hat{\Params}$ need to be for \Adv{} to distinguish two datasets with high probability.
\item We provide examples of conjugate-pair distributions where our assumptions hold, including discrete Bayesian networks.
\end{itemize}

\emph{Paper organisation.}
Section~\ref{sec:setting} specifies the setting and our assumptions.
 Section~\ref{sec:robustness} proves results on robustness of Bayesian learning. Section~\ref{sec:privacy} proves our main privacy results. In particular, Section~\ref{sec:diff-priv-post} shows that the posterior distribution is differentially private, Section~\ref{sec:post-sampl-query} describes our posterior sampling query response algorithm, Section~\ref{sec:bound-post-query} derives bounds on dataset indistinguishability, Section~\ref{app:le-cam} shows how to obtain matching lower bounds for distinguishability, while Section~\ref{sec:trad-util-priv} shows how utility and privacy can be traded off within our framework.
Examples where our assumptions hold are given in Section~\ref{sec:examples}. We present a discussion of our results, related work and links to the exponential mechanism and robust Bayesian inference in Section~\ref{sec:conclusion}.
Appendix~\ref{sec:proofs} contains proofs of the main theorems.
Finally, Appendix~\ref{sec:example-proofs} details proofs of the examples demonstrating our assumptions.

\section{Problem Setting}
\label{sec:setting}
We consider the problem of a Bayesian statistician (\Bay{}) communicating with an untrusted third party (\Adv{}). \Bay{} wants
to convey useful responses to the queries of \Adv{} (\eg how many people suffer from a disease or vote for a particular party) without revealing private information about the original data (\eg whether a particular person has cancer).  This requires communicating information in a way that strikes a balance between utility and privacy. In this paper, we study the inherent privacy and robustness properties of Bayesian inference and explore the question of whether \Bay{} can select a prior distribution so that a computationally unbounded \Adv{} cannot obtain private information from queries.

\subsection{Definitions and Notation}

We begin with our notation. Let $\CS$ be the set of all possible
datasets.  For example, if $\CX$ is a finite alphabet, then we might
have $\CS = \bigcup_{n=0}^\infty \CX^n$, \ie the set of all possible
observation sequences over $\CX$. However, $\CS$ can have arbitrary
structure and so social network or mobility trace data are also
handled in this framework.  Probability measures on parameters $\param$ are usually
denoted by $\bel$, while measures and densities on data are denoted by
$P_\param$ or $p_\param$ respectively. Expectations are denoted by
$\E_\bel g \defn \int_\Params g(\param) \dd \bel(\param)$, where the
subscript denotes the underlying distribution with respect to which we
are taking expectations. In case of ambiguity, we explicitly write \eg
$\E_{x \sim P_\param} f(x) = \int_\CS f(x) \dd P_\param(x)$ to denote
which variables are drawn from which distributions.  Finally, we use
$\ind{\pi}$ to be the identity function, taking the value $1$ when the
predicate $\pi$ is true, and $0$ otherwise.

\subsubsection{Distances Between Datasets} Central to the notions of privacy and robustness, is the concept of distance between datasets. Firstly, the effect of dataset perturbation on learning depends on the amount of noise as quantified by some distance. This is useful for characterising robustness to noise or adversarial manipulation of the data. Secondly, the amount that an attacker can learn from queries can be quantified in terms of the distance of his guesses to the true dataset. Finally, it allows for a unified mathematical treatment, as it permits different types of neighbourhoods to be defined. To model these situations, we equip $\CS$ with a pseudo-metric\footnote{Meaning that $\xdist{x}{y} = 0$ does not necessarily imply $x = y$.} $\xdistChar : \CS \times \CS \to \Reals_+$. This generalisation has also been used by~\citet{PET:DP}, which has laid the groundwork for metric-based differential privacy. While this concept has many applications in the context of geographical information systems, we apply this generalisation of differential privacy without necessarily referring to some underlying physical distance.

%This is required for a notion of similarity between datasets. Intuitively, two datasets $x, y$ with small distance $\xdist{x}{y}$ will result in similar outputs for the learning algorithm.

\subsubsection{Bayesian Inference}
This paper focuses on the \emph{Bayesian inference} setting, where the statistician $\Bay{}$ constructs a posterior distribution from a prior distribution $\bel$ and a training dataset $x$. More precisely, we assume that data $x \in \CS$ have been drawn from some distribution $P_{\theta^\star}$ on $\CS$, parameterised by $\theta^\star$, from a family of distributions $\family$.
$\Bay{}$ defines a parameter set $\Params$ indexing a family of distributions $\family$ on $(\CS, \field{\CS})$, where $\field{\CS}$ is an appropriate $\sigma$-algebra on $\CS$:
\begin{equation*}
  \family \defn \cset{P_\param}{\param \in \Params},
  %\label{eq:family}
\end{equation*}
and where we use $p_\param$ to denote the corresponding densities\footnote{\Ie the Radon-Nikodym derivative of $P_\param$ relative to some dominating measure $\dom$.} when necessary.
To perform inference in the Bayesian setting, $\Bay{}$ selects a prior measure $\bel$ on $(\Params, \field{\Params})$ reflecting \Bay{}'s subjective beliefs about which $\theta$ is more likely to be true, {\em a priori}; \ie for any measurable set $B \in \field{\Params}$, $\bel(B)$ represents \Bay{}'s prior belief that $\theta^\star \in B$.
In general, the posterior distribution after observing $x \in \CS$ is:
\begin{align}
  \label{eq:posterior}
  \bel(B \mid x) = \frac{\int_B p_\param(x) \dd{\bel}(\param)}{\marg(x)}
%, &\qquad
%  &&
%  x \in \CS
  \enspace
  ,
\end{align}
where $\marg$ is the corresponding marginal density given by:
\begin{align*}
  \marg(x)
  & \defn \int_\Params p_\param(x) \dd{\bel}(\param)%,
%  &&
%  x \in \CS
  \enspace.
\end{align*}
While the choice of the prior is generally arbitrary, this paper shows that its careful selection can yield good privacy guarantees. Throughout the paper, we shall use the following simple example to ground our observations and theory. This consists of a finite family of distributions, on a finite alphabet. Consequently, calculation of the posterior distribution is always simple. It is also easy to verify our assumptions on this model.

\begin{example}[Finite Bernoulli family.]
Consider a finite family of distributions $\family = \cset{P_\param}{\param \in \Params}$ on alphabet $\CX = \set{0, 1}$, with $\param \in [0,1]$, such that for any model in the family and any observation $x$
\[
P_\param(x) = \param \ind{x = 1} + (1 - \param) \ind{x = 0}.
\]
For any sequence of observations $x_1, \ldots, x_T$, we have, with some abuse of notation,
\[
P_\param(x_1, \ldots, x_T) = \prod_{t=1}^T P_\param(x_t),
\]
\ie $P_\param$ defines an i.i.d. distribution on the alphabet.
\label{ex:finite}
This family corresponds to a set of Bernoulli models.  The set of
parameters $\Params$ will be chosen to discretise the parameter space
of Bernoulli distributions over $\Delta$-sized intervals. For this, the $k$-th
model's parameter will be
$\theta_k = \Delta k$,
with $\Delta \in (0,1)$
and $k \leq 1/\Delta$.
\end{example}
For the above family, we can use a uniform prior distribution $\bel(\theta_k) = \Delta$. The posterior distribution is easily calculated, since we need only sum over a finite number of parameters.

\subsubsection{Privacy} We now recall the concept of differential privacy~\citep{DworkOnly06}. This states that on \emph{neighbouring} datasets, a randomised query response mechanism yields (pointwise) similar distributions.
We adopt the view of mechanisms as conditional distributions %---similar to \citet{duchi:local-privacy}---
under which differential privacy
can be seen as a measure of smoothness.
In our setting, conditional distributions conveniently correspond to posterior distributions. These can also be interpreted as the distribution of a mechanism that uses posterior sampling, to be introduced in Section~\ref{sec:post-sampl-query}.
The precise definition depends on the notion of neighbourhood, with the following choice being common:
\begin{definition}[$(\epsilon, \delta)$-differential privacy]
\label{def:dp-old}
A conditional distribution $P(\cdot \mid x)$ on $(\Params, \field{\Params})$ is \emph{$(\epsilon, \delta)$-differentially private} if, for all $B \in \field{\Params}$ and for any $x \in \CS = \CX^n$
\[
  P(B \mid x) \leq e^\epsilon P(B \mid y) + \delta,
\]
for all $y$ in the \emph{Hamming-$1$ neighbourhood} of $x$. That is,
$y$ may differ in at most one entry from $x$:
there is at most one $i \in \{1, \ldots, n\}$ such that $x_i \neq y_i$.
\end{definition}
A typical situation where this definition is employed, is when $x, y$ are matrices and $x_i$ is a single row in the matrix. Then, the datasets are neighbours if a matrix row is changed.\footnote{Another common choice for neighbourhoods is to say that two datasets are neighbours if one results from the other by addition of a row.}

In our setting, it is reasonable to generalise this to arbitrary dataset spaces $\CS$ that are not necessarily product spaces. To do so, we use the notion of differential privacy under a pseudo-metric $\xdistChar$ on the space of all datasets, which allows for more subtle representations of attacker knowledge and for a more general treatment:
\begin{definition}[$(\epsilon, \delta)$-differential privacy under $\xdistChar$.]
\label{def:dp-new}
A conditional distribution $P(\cdot \mid x)$ on $(\Params, \field{\Params})$ is \emph{$(\epsilon, \delta)$-differentially private under} a pseudo-metric $\xdistChar : \CS \times \CS \to \Reals_+$ if, for all $B \in \field{\Params}$ and for any $x, y \in \CS$,
\[
  P(B \mid x) \leq e^{\epsilon \xdist{x}{y}} P(B \mid y) + \delta \xdist{x}{y}\;.
\]
\end{definition}
In our setting, $\xdistChar$ replaces the notion of neighbourhood. It is of course possible to use $\xdistChar$ that corresponds to the usual meaning of neighbourhood in differential privacy:
\begin{remark}
  If $\CS = \CX^n$ and $\xdist{x}{y} = \sum_{i=1}^n \ind{x_i \neq y_i}$ is the Hamming distance, this definition is analogous to standard $(\epsilon, \delta)$-differential privacy. When considering only $(\epsilon, 0)$- differential privacy or $(0,\delta)$-privacy, it is an equivalent notion.\footnote{Making the definition wholly equivalent is possible, but results in an unnecessarily complex definition.}
\end{remark}
\begin{proof}
  For $(\epsilon, 0)$-DP, let $\rho(x,z) = \rho(z,y) = 1$; \ie the data differ in one element. Then, from standard DP, we have
  $P(B \mid x) \leq e^{\epsilon} P(B \mid z)$ and so obtain
  $P(B \mid x) \leq e^{2 \epsilon} P(B \mid y) = e^{\rho(x, y) \epsilon} P(B \mid y)$. By induction, this holds for any $x,y$ pair.
  Similarly, for $(0, \delta)$-DP, by induction we obtain $P(B \mid x) \leq P(B \mid y) + \delta \rho(x, y)$.
\end{proof}
Definition~\ref{def:dp-old} allows for privacy against a powerful attacker \Adv{}, who attempts to match the empirical distribution induced by the true dataset,
by querying the learned mechanism and comparing its responses to those given by distributions simulated using knowledge of the mechanism
and knowledge of all but one datum---narrowing the dataset down to a Hamming-1 ball. Indeed the requirement of differential privacy is sometimes {\em too strong} since it may come at the price of utility. Definition~\ref{def:dp-new} allows for a much broader encoding of the attacker's knowledge via the selected pseudo-metric. It also allows a more fine-grained notion of privacy. This is quite useful for geographical information systems, as proposed by~\citet{PET:DP}, to which we refer the reader for a broader discussion of the use of metrics in differential privacy.

Finally, we can show that this generalisation of differential privacy satisfies the standard composition property.
\begin{theorem}[Composition]
Let conditional distributions $P(\cdot \mid x)$ on  $(\Params, \field{\Params})$  be \emph{$(\epsilon, \delta)$-differentially private under} a pseudo-metric $\xdistChar: \CS \times \CS \to \Reals_+$ and $P'(\cdot \mid x)$ on $(\Params', \field{\Params'}')$ be \emph{$(\epsilon', \delta')$-differentially private under} the same pseudo-metric. Then the conditional distribution on the product space $(\Params \times \Params', \field{\Params}\otimes \field{\Params'}' )$ given by
$$Q(B \times B' \mid x)= P(B \mid x) P'(B' \mid x), \forall B \times B' \in \field{\Params}\otimes \field{\Params'}'$$
satisfies \emph{$(\epsilon+\epsilon', \delta+\delta')$-differentially private under} the pseudo-metric $\xdistChar$. Here $\field{\Params}\otimes \field{\Params'}' $ is the product $\sigma$-algebra on $\Params \times \Params'$.
\label{thm:composition}
\end{theorem}
\begin{proof} For any $y\in\CS$
\begin{align*}
Q(B \times B' \mid x)&\leq \left[e^{\epsilon \xdist{x}{y}}P(B \mid y)+\delta\xdist{x}{y}\right]P'(B' \mid x)
\\ \nonumber
 &\leq
 e^{\epsilon \xdist{x}{y}}P(B \mid y)\left[e^{\epsilon' \xdist{x}{y}}P'(B' \mid y)+\delta' \xdist{x}{y}\right]+\delta \xdist{x}{y}
 \\ \nonumber
 &\leq e^{(\epsilon+\epsilon') \xdist{x}{y}}P(B \mid y)P'(B' \mid y)+(\delta+\delta')\xdist{x}{y}
\end{align*}
\end{proof}

% Other properties?
% % what is the thing below for?
%   \begin{align}
%     \Pr(\onenorm{\hat{p}_n - p} & \geq  \sqrt{\frac{3}{n} \ln \frac{1}{\delta}}) \leq
%     e^{m \ln 2 - \frac{3}{2} \ln \frac{1}{\delta}}
%     \\ \nonumber
%     &\leq
%     e^{\log_2 \sqrt{\frac{1}{\delta}} \ln 2 - \frac{3}{2} \ln \frac{1}{\delta}}
%     =
%     e^{\frac{1}{2} \ln \frac{1}{\delta} - \frac{3}{2} \ln \frac{1}{\delta}} = \delta.
%   \end{align}
\subsection{Our Main Assumptions}
\label{sec:framework}
In the sequel, we show that if the distribution family $\family$ or prior $\bel$ satisfies certain assumptions, then close datasets $x, y \in \CS$ result in posterior distributions that are close. In that case, it is difficult for a third party to use such a posterior to distinguish the true dataset $x$ from similar datasets.

To formalise these notions, we introduce two possible assumptions one
could make on the smoothness of the family $\family$ with respect to
some metric $d$ on $\Reals_+$. The first assumption states that the
likelihood is smooth for all parameterisations of the family.  First,
we define our notion of smoothness.  Let
$f(x,\param) \defn \ln p_\param(x)$ be the log probability of $x$
under $\param$. The Lipschitz constant for a parameter value $\param$
is:
\begin{equation}
  \label{eq:Lipschitz}
  \ell(\param) \defn \inf \cset{u}{|f(x,\param) - f(y,\param)| \leq u \xdist{x}{y} \forall x, y \in \CS}.
\end{equation}
Our first assumption is uniform smoothness for all parameters.
\begin{assumption}[Lipschitz continuity]
We assume there exists some $L < \infty$ such that:
  \begin{align}
    \ell(\param) \leq L,
    \qquad
    \param \in \Params.
    \label{eq:hoelder-observations}
  \end{align}
  \label{ass:hoelder-observations}
\end{assumption}
In other words, this assumption says that the log probability is Lipschitz with respect to $\xdistChar$ for any parameter value. Consider Example~\ref{ex:finite} for the Bernoulli model. It is easy to see that a model with $\Delta$-sized intervals satisfies the above assumption with $L = \ln 1/\Delta$.

 However, it may be difficult for this assumption to hold uniformly over $\Params$ in general. This can be seen by the following counterexample for the Bernoulli family of distributions: when the parameter is $0$, then any sequence $x = 0, 0, \ldots$ has probability $1$, while any sequence containing a $1$ has probability $0$. The same thing occurs when we take $\Delta \to 0$ in Example~\ref{ex:finite}. To avoid such problems, we relax the assumption by only requiring that \Bay{}'s \emph{prior} probability $\bel$ is concentrated in the regions of the family for which the likelihood is smoothest:
\begin{assumption}[Stochastic Lipschitz continuity;~\citealp{norkin1986stochastic}]
First, define the subset of parameter values
\begin{eqnarray}
  \Params_L   &\defn
\cset{\param \in \Params}{\ell(\param) \leq L}
    \label{eq:hoelder-measure-set}
  \end{eqnarray}
  to be those parameters for which Lipschitz continuity holds with Lipschitz constant $L$.
   Then, there are some constants $c, L_{0} > 0$ such that, for all $L \geq L_{0}$:
  \begin{equation}
    \bel(\Params_L) \geq 1 - \exp(-c(L-L_{0}))\enspace.
    \label{eq:hoelder-measure-observations}
  \end{equation}
  \label{ass:hoelder-measure-observations}
\end{assumption}
By not requiring uniform smoothness, this weaker assumption is easier to meet but still yields useful guarantees. In fact, in Section~\ref{sec:examples}, we demonstrate that this assumption is satisfied by many important example distribution families. However, it will be illustrative to consider the discrete Bernoulli family example at this point.

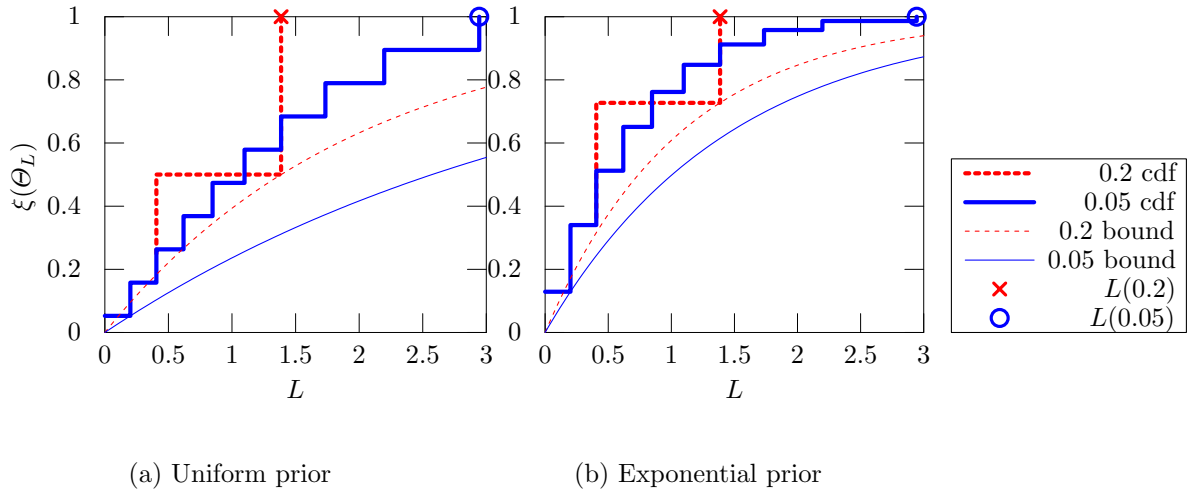
\begin{figure}[t]
%  \centering
  \begin{subfigure}[b]{0.4\textwidth}
    \begin{tikzpicture}[gnuplot]
%% generated with GNUPLOT 4.6p6 (Lua 5.1; terminal rev. 99, script rev. 100)
%% Fri Dec 16 15:52:14 2016
\path (0.000,0.000) rectangle (6.944,5.556);
\gpfill{rgb color={1.000,1.000,1.000}} (1.320,0.985)--(6.390,0.985)--(6.390,5.185)--(1.320,5.185)--cycle;
\gpcolor{color=gp lt color border}
\gpsetlinetype{gp lt border}
\gpsetlinewidth{0.50}
\draw[gp path] (1.320,0.985)--(1.571,0.985);
\draw[gp path] (6.391,0.985)--(6.140,0.985);
\gpcolor{rgb color={0.000,0.000,0.000}}
\node[gp node right,font={\fontsize{10pt}{12pt}\selectfont}] at (1.136,0.985) {0};
\gpcolor{color=gp lt color border}
\draw[gp path] (1.320,1.825)--(1.571,1.825);
\draw[gp path] (6.391,1.825)--(6.140,1.825);
\gpcolor{rgb color={0.000,0.000,0.000}}
\node[gp node right,font={\fontsize{10pt}{12pt}\selectfont}] at (1.136,1.825) {0.2};
\gpcolor{color=gp lt color border}
\draw[gp path] (1.320,2.665)--(1.571,2.665);
\draw[gp path] (6.391,2.665)--(6.140,2.665);
\gpcolor{rgb color={0.000,0.000,0.000}}
\node[gp node right,font={\fontsize{10pt}{12pt}\selectfont}] at (1.136,2.665) {0.4};
\gpcolor{color=gp lt color border}
\draw[gp path] (1.320,3.506)--(1.571,3.506);
\draw[gp path] (6.391,3.506)--(6.140,3.506);
\gpcolor{rgb color={0.000,0.000,0.000}}
\node[gp node right,font={\fontsize{10pt}{12pt}\selectfont}] at (1.136,3.506) {0.6};
\gpcolor{color=gp lt color border}
\draw[gp path] (1.320,4.346)--(1.571,4.346);
\draw[gp path] (6.391,4.346)--(6.140,4.346);
\gpcolor{rgb color={0.000,0.000,0.000}}
\node[gp node right,font={\fontsize{10pt}{12pt}\selectfont}] at (1.136,4.346) {0.8};
\gpcolor{color=gp lt color border}
\draw[gp path] (1.320,5.186)--(1.571,5.186);
\draw[gp path] (6.391,5.186)--(6.140,5.186);
\gpcolor{rgb color={0.000,0.000,0.000}}
\node[gp node right,font={\fontsize{10pt}{12pt}\selectfont}] at (1.136,5.186) {1};
\gpcolor{color=gp lt color border}
\draw[gp path] (1.320,0.985)--(1.320,1.236);
\draw[gp path] (1.320,5.186)--(1.320,4.935);
\gpcolor{rgb color={0.000,0.000,0.000}}
\node[gp node center,font={\fontsize{10pt}{12pt}\selectfont}] at (1.320,0.677) {0};
\gpcolor{color=gp lt color border}
\draw[gp path] (2.165,0.985)--(2.165,1.236);
\draw[gp path] (2.165,5.186)--(2.165,4.935);
\gpcolor{rgb color={0.000,0.000,0.000}}
\node[gp node center,font={\fontsize{10pt}{12pt}\selectfont}] at (2.165,0.677) {0.5};
\gpcolor{color=gp lt color border}
\draw[gp path] (3.010,0.985)--(3.010,1.236);
\draw[gp path] (3.010,5.186)--(3.010,4.935);
\gpcolor{rgb color={0.000,0.000,0.000}}
\node[gp node center,font={\fontsize{10pt}{12pt}\selectfont}] at (3.010,0.677) {1};
\gpcolor{color=gp lt color border}
\draw[gp path] (3.856,0.985)--(3.856,1.236);
\draw[gp path] (3.856,5.186)--(3.856,4.935);
\gpcolor{rgb color={0.000,0.000,0.000}}
\node[gp node center,font={\fontsize{10pt}{12pt}\selectfont}] at (3.856,0.677) {1.5};
\gpcolor{color=gp lt color border}
\draw[gp path] (4.701,0.985)--(4.701,1.236);
\draw[gp path] (4.701,5.186)--(4.701,4.935);
\gpcolor{rgb color={0.000,0.000,0.000}}
\node[gp node center,font={\fontsize{10pt}{12pt}\selectfont}] at (4.701,0.677) {2};
\gpcolor{color=gp lt color border}
\draw[gp path] (5.546,0.985)--(5.546,1.236);
\draw[gp path] (5.546,5.186)--(5.546,4.935);
\gpcolor{rgb color={0.000,0.000,0.000}}
\node[gp node center,font={\fontsize{10pt}{12pt}\selectfont}] at (5.546,0.677) {2.5};
\gpcolor{color=gp lt color border}
\draw[gp path] (6.391,0.985)--(6.391,1.236);
\draw[gp path] (6.391,5.186)--(6.391,4.935);
\gpcolor{rgb color={0.000,0.000,0.000}}
\node[gp node center,font={\fontsize{10pt}{12pt}\selectfont}] at (6.391,0.677) {3};
\gpcolor{color=gp lt color border}
\draw[gp path] (1.320,5.186)--(1.320,0.985)--(6.391,0.985)--(6.391,5.186)--cycle;
\gpcolor{rgb color={0.000,0.000,0.000}}
\node[gp node center,rotate=90,font={\fontsize{10pt}{12pt}\selectfont}] at (0.246,3.085) {$\bel(\Params_L)$};
\node[gp node center,font={\fontsize{10pt}{12pt}\selectfont}] at (3.855,0.215) {$L$};
\gpcolor{rgb color={1.000,0.000,0.000}}
\gpsetlinetype{gp lt plot 2}
\gpsetlinewidth{4.00}
\draw[gp path] (2.005,2.035)--(2.005,3.086)--(3.663,3.086)--(3.663,4.136)--(3.663,5.186);
\gpcolor{rgb color={0.000,0.000,1.000}}
\gpsetlinetype{gp lt plot 0}
\draw[gp path] (1.320,1.206)--(1.659,1.206)--(1.659,1.427)--(1.659,1.648)--(2.005,1.648)%
  --(2.005,1.869)--(2.005,2.091)--(2.366,2.091)--(2.366,2.312)--(2.366,2.533)--(2.752,2.533)%
  --(2.752,2.754)--(2.752,2.975)--(3.177,2.975)--(3.177,3.417)--(3.663,3.417)--(3.663,3.638)%
  --(3.663,3.859)--(4.252,3.859)--(4.252,4.080)--(4.252,4.302)--(5.034,4.302)--(5.034,4.523)%
  --(5.034,4.744)--(6.297,4.744)--(6.297,4.965)--(6.297,5.186);
\gpcolor{rgb color={1.000,0.000,0.000}}
\gpsetlinetype{gp lt plot 2}
\gpsetlinewidth{1.00}
\draw[gp path] (1.320,0.985)--(1.371,1.048)--(1.421,1.109)--(1.472,1.170)--(1.523,1.230)%
  --(1.574,1.289)--(1.624,1.347)--(1.675,1.404)--(1.726,1.460)--(1.776,1.516)--(1.827,1.570)%
  --(1.878,1.624)--(1.929,1.677)--(1.979,1.729)--(2.030,1.781)--(2.081,1.831)--(2.131,1.881)%
  --(2.182,1.931)--(2.233,1.979)--(2.283,2.027)--(2.334,2.074)--(2.385,2.120)--(2.436,2.166)%
  --(2.486,2.211)--(2.537,2.255)--(2.588,2.299)--(2.638,2.342)--(2.689,2.384)--(2.740,2.426)%
  --(2.791,2.467)--(2.841,2.507)--(2.892,2.547)--(2.943,2.586)--(2.993,2.625)--(3.044,2.663)%
  --(3.095,2.701)--(3.146,2.738)--(3.196,2.774)--(3.247,2.810)--(3.298,2.846)--(3.348,2.880)%
  --(3.399,2.915)--(3.450,2.949)--(3.501,2.982)--(3.551,3.015)--(3.602,3.047)--(3.653,3.079)%
  --(3.703,3.110)--(3.754,3.141)--(3.805,3.172)--(3.856,3.202)--(3.906,3.231)--(3.957,3.260)%
  --(4.008,3.289)--(4.058,3.317)--(4.109,3.345)--(4.160,3.372)--(4.210,3.399)--(4.261,3.426)%
  --(4.312,3.452)--(4.363,3.478)--(4.413,3.503)--(4.464,3.528)--(4.515,3.553)--(4.565,3.577)%
  --(4.616,3.601)--(4.667,3.625)--(4.718,3.648)--(4.768,3.671)--(4.819,3.694)--(4.870,3.716)%
  --(4.920,3.738)--(4.971,3.759)--(5.022,3.781)--(5.073,3.802)--(5.123,3.822)--(5.174,3.842)%
  --(5.225,3.862)--(5.275,3.882)--(5.326,3.902)--(5.377,3.921)--(5.428,3.940)--(5.478,3.958)%
  --(5.529,3.976)--(5.580,3.994)--(5.630,4.012)--(5.681,4.030)--(5.732,4.047)--(5.782,4.064)%
  --(5.833,4.080)--(5.884,4.097)--(5.935,4.113)--(5.985,4.129)--(6.036,4.145)--(6.087,4.160)%
  --(6.137,4.176)--(6.188,4.191)--(6.239,4.205)--(6.290,4.220)--(6.340,4.234)--(6.391,4.249);
\gpcolor{rgb color={0.000,0.000,1.000}}
\gpsetlinetype{gp lt plot 0}
\draw[gp path] (1.320,0.985)--(1.371,1.019)--(1.421,1.052)--(1.472,1.086)--(1.523,1.119)%
  --(1.574,1.151)--(1.624,1.184)--(1.675,1.216)--(1.726,1.248)--(1.776,1.280)--(1.827,1.311)%
  --(1.878,1.342)--(1.929,1.373)--(1.979,1.404)--(2.030,1.434)--(2.081,1.465)--(2.131,1.495)%
  --(2.182,1.524)--(2.233,1.554)--(2.283,1.583)--(2.334,1.612)--(2.385,1.641)--(2.436,1.669)%
  --(2.486,1.698)--(2.537,1.726)--(2.588,1.754)--(2.638,1.781)--(2.689,1.809)--(2.740,1.836)%
  --(2.791,1.863)--(2.841,1.890)--(2.892,1.916)--(2.943,1.942)--(2.993,1.969)--(3.044,1.994)%
  --(3.095,2.020)--(3.146,2.046)--(3.196,2.071)--(3.247,2.096)--(3.298,2.121)--(3.348,2.146)%
  --(3.399,2.170)--(3.450,2.194)--(3.501,2.218)--(3.551,2.242)--(3.602,2.266)--(3.653,2.289)%
  --(3.703,2.313)--(3.754,2.336)--(3.805,2.359)--(3.856,2.382)--(3.906,2.404)--(3.957,2.427)%
  --(4.008,2.449)--(4.058,2.471)--(4.109,2.493)--(4.160,2.514)--(4.210,2.536)--(4.261,2.557)%
  --(4.312,2.578)--(4.363,2.599)--(4.413,2.620)--(4.464,2.641)--(4.515,2.661)--(4.565,2.682)%
  --(4.616,2.702)--(4.667,2.722)--(4.718,2.742)--(4.768,2.761)--(4.819,2.781)--(4.870,2.800)%
  --(4.920,2.819)--(4.971,2.839)--(5.022,2.857)--(5.073,2.876)--(5.123,2.895)--(5.174,2.913)%
  --(5.225,2.931)--(5.275,2.950)--(5.326,2.968)--(5.377,2.986)--(5.428,3.003)--(5.478,3.021)%
  --(5.529,3.038)--(5.580,3.056)--(5.630,3.073)--(5.681,3.090)--(5.732,3.107)--(5.782,3.123)%
  --(5.833,3.140)--(5.884,3.156)--(5.935,3.173)--(5.985,3.189)--(6.036,3.205)--(6.087,3.221)%
  --(6.137,3.237)--(6.188,3.252)--(6.239,3.268)--(6.290,3.283)--(6.340,3.299)--(6.391,3.314);
\gpcolor{rgb color={1.000,0.000,0.000}}
\gpsetlinewidth{3.00}
\gpsetpointsize{8.00}
\gppoint{gp mark 2}{(3.663,5.186)}
\gpcolor{rgb color={0.000,0.000,1.000}}
\gppoint{gp mark 6}{(6.297,5.186)}
%% coordinates of the plot area
\gpdefrectangularnode{gp plot 1}{\pgfpoint{1.320cm}{0.985cm}}{\pgfpoint{6.391cm}{5.186cm}}
\end{tikzpicture}
%% gnuplot variables
    \caption{Uniform prior}
    \label{fig:finite-uniform}
  \end{subfigure}
  \begin{subfigure}[b]{0.4\textwidth}
    \begin{tikzpicture}[gnuplot]
%% generated with GNUPLOT 4.6p6 (Lua 5.1; terminal rev. 99, script rev. 100)
%% Fri Dec 16 15:55:51 2016
\path (0.000,0.000) rectangle (9.722,5.556);
\gpfill{rgb color={1.000,1.000,1.000}} (1.012,0.985)--(6.044,0.985)--(6.044,5.185)--(1.012,5.185)--cycle;
\gpcolor{color=gp lt color border}
\gpsetlinetype{gp lt border}
\gpsetlinewidth{0.50}
\draw[gp path] (1.012,0.985)--(1.263,0.985);
\draw[gp path] (6.045,0.985)--(5.794,0.985);
\gpcolor{rgb color={0.000,0.000,0.000}}
\node[gp node right,font={\fontsize{10pt}{12pt}\selectfont}] at (0.828,0.985) {0};
\gpcolor{color=gp lt color border}
\draw[gp path] (1.012,1.825)--(1.263,1.825);
\draw[gp path] (6.045,1.825)--(5.794,1.825);
\gpcolor{rgb color={0.000,0.000,0.000}}
\node[gp node right,font={\fontsize{10pt}{12pt}\selectfont}] at (0.828,1.825) {0.2};
\gpcolor{color=gp lt color border}
\draw[gp path] (1.012,2.665)--(1.263,2.665);
\draw[gp path] (6.045,2.665)--(5.794,2.665);
\gpcolor{rgb color={0.000,0.000,0.000}}
\node[gp node right,font={\fontsize{10pt}{12pt}\selectfont}] at (0.828,2.665) {0.4};
\gpcolor{color=gp lt color border}
\draw[gp path] (1.012,3.506)--(1.263,3.506);
\draw[gp path] (6.045,3.506)--(5.794,3.506);
\gpcolor{rgb color={0.000,0.000,0.000}}
\node[gp node right,font={\fontsize{10pt}{12pt}\selectfont}] at (0.828,3.506) {0.6};
\gpcolor{color=gp lt color border}
\draw[gp path] (1.012,4.346)--(1.263,4.346);
\draw[gp path] (6.045,4.346)--(5.794,4.346);
\gpcolor{rgb color={0.000,0.000,0.000}}
\node[gp node right,font={\fontsize{10pt}{12pt}\selectfont}] at (0.828,4.346) {0.8};
\gpcolor{color=gp lt color border}
\draw[gp path] (1.012,5.186)--(1.263,5.186);
\draw[gp path] (6.045,5.186)--(5.794,5.186);
\gpcolor{rgb color={0.000,0.000,0.000}}
\node[gp node right,font={\fontsize{10pt}{12pt}\selectfont}] at (0.828,5.186) {1};
\gpcolor{color=gp lt color border}
\draw[gp path] (1.012,0.985)--(1.012,1.236);
\draw[gp path] (1.012,5.186)--(1.012,4.935);
\gpcolor{rgb color={0.000,0.000,0.000}}
\node[gp node center,font={\fontsize{10pt}{12pt}\selectfont}] at (1.012,0.677) {0};
\gpcolor{color=gp lt color border}
\draw[gp path] (1.851,0.985)--(1.851,1.236);
\draw[gp path] (1.851,5.186)--(1.851,4.935);
\gpcolor{rgb color={0.000,0.000,0.000}}
\node[gp node center,font={\fontsize{10pt}{12pt}\selectfont}] at (1.851,0.677) {0.5};
\gpcolor{color=gp lt color border}
\draw[gp path] (2.690,0.985)--(2.690,1.236);
\draw[gp path] (2.690,5.186)--(2.690,4.935);
\gpcolor{rgb color={0.000,0.000,0.000}}
\node[gp node center,font={\fontsize{10pt}{12pt}\selectfont}] at (2.690,0.677) {1};
\gpcolor{color=gp lt color border}
\draw[gp path] (3.529,0.985)--(3.529,1.236);
\draw[gp path] (3.529,5.186)--(3.529,4.935);
\gpcolor{rgb color={0.000,0.000,0.000}}
\node[gp node center,font={\fontsize{10pt}{12pt}\selectfont}] at (3.529,0.677) {1.5};
\gpcolor{color=gp lt color border}
\draw[gp path] (4.367,0.985)--(4.367,1.236);
\draw[gp path] (4.367,5.186)--(4.367,4.935);
\gpcolor{rgb color={0.000,0.000,0.000}}
\node[gp node center,font={\fontsize{10pt}{12pt}\selectfont}] at (4.367,0.677) {2};
\gpcolor{color=gp lt color border}
\draw[gp path] (5.206,0.985)--(5.206,1.236);
\draw[gp path] (5.206,5.186)--(5.206,4.935);
\gpcolor{rgb color={0.000,0.000,0.000}}
\node[gp node center,font={\fontsize{10pt}{12pt}\selectfont}] at (5.206,0.677) {2.5};
\gpcolor{color=gp lt color border}
\draw[gp path] (6.045,0.985)--(6.045,1.236);
\draw[gp path] (6.045,5.186)--(6.045,4.935);
\gpcolor{rgb color={0.000,0.000,0.000}}
\node[gp node center,font={\fontsize{10pt}{12pt}\selectfont}] at (6.045,0.677) {3};
\gpcolor{color=gp lt color border}
\draw[gp path] (1.012,5.186)--(1.012,0.985)--(6.045,0.985)--(6.045,5.186)--cycle;
\gpcolor{rgb color={0.000,0.000,0.000}}
\node[gp node center,font={\fontsize{10pt}{12pt}\selectfont}] at (3.528,0.215) {$L$};
\gpcolor{color=gp lt color border}
\gpsetlinewidth{1.00}
\draw[gp path] (6.413,0.985)--(6.413,3.295)--(9.537,3.295)--(9.537,0.985)--cycle;
\gpcolor{rgb color={0.000,0.000,0.000}}
\node[gp node right,font={\fontsize{10pt}{12pt}\selectfont}] at (9.537,3.103) {0.2 cdf};
\gpcolor{rgb color={1.000,0.000,0.000}}
\gpsetlinetype{gp lt plot 2}
\gpsetlinewidth{4.00}
\draw[gp path] (6.597,3.103)--(7.513,3.103);
\draw[gp path] (1.692,2.513)--(1.692,4.040)--(3.338,4.040)--(3.338,4.613)--(3.338,5.186);
\gpcolor{rgb color={0.000,0.000,0.000}}
\node[gp node right,font={\fontsize{10pt}{12pt}\selectfont}] at (9.537,2.718) {0.05 cdf};
\gpcolor{rgb color={0.000,0.000,1.000}}
\gpsetlinetype{gp lt plot 0}
\draw[gp path] (6.597,2.718)--(7.513,2.718);
\draw[gp path] (1.012,1.527)--(1.349,1.527)--(1.349,1.970)--(1.349,2.414)--(1.692,2.414)%
  --(1.692,2.775)--(1.692,3.137)--(2.051,3.137)--(2.051,3.428)--(2.051,3.720)--(2.433,3.720)%
  --(2.433,3.953)--(2.433,4.185)--(2.855,4.185)--(2.855,4.546)--(3.338,4.546)--(3.338,4.682)%
  --(3.338,4.817)--(3.922,4.817)--(3.922,4.913)--(3.922,5.009)--(4.698,5.009)--(4.698,5.069)%
  --(4.698,5.129)--(5.952,5.129)--(5.952,5.157)--(5.952,5.186);
\gpcolor{rgb color={0.000,0.000,0.000}}
\node[gp node right,font={\fontsize{10pt}{12pt}\selectfont}] at (9.537,2.333) {0.2 bound};
\gpcolor{rgb color={1.000,0.000,0.000}}
\gpsetlinetype{gp lt plot 2}
\gpsetlinewidth{1.00}
\draw[gp path] (6.597,2.333)--(7.513,2.333);
\draw[gp path] (1.012,0.985)--(1.062,1.101)--(1.113,1.215)--(1.163,1.325)--(1.213,1.432)%
  --(1.264,1.536)--(1.314,1.637)--(1.364,1.736)--(1.415,1.831)--(1.465,1.924)--(1.515,2.015)%
  --(1.566,2.103)--(1.616,2.188)--(1.666,2.271)--(1.717,2.352)--(1.767,2.431)--(1.817,2.507)%
  --(1.868,2.581)--(1.918,2.653)--(1.968,2.724)--(2.019,2.792)--(2.069,2.858)--(2.119,2.923)%
  --(2.170,2.986)--(2.220,3.047)--(2.270,3.106)--(2.321,3.164)--(2.371,3.220)--(2.421,3.274)%
  --(2.472,3.327)--(2.522,3.379)--(2.572,3.429)--(2.623,3.478)--(2.673,3.525)--(2.723,3.571)%
  --(2.774,3.616)--(2.824,3.659)--(2.874,3.702)--(2.925,3.743)--(2.975,3.783)--(3.025,3.822)%
  --(3.076,3.860)--(3.126,3.896)--(3.176,3.932)--(3.227,3.967)--(3.277,4.001)--(3.327,4.033)%
  --(3.378,4.065)--(3.428,4.097)--(3.478,4.127)--(3.529,4.156)--(3.579,4.185)--(3.629,4.212)%
  --(3.679,4.239)--(3.730,4.266)--(3.780,4.291)--(3.830,4.316)--(3.881,4.340)--(3.931,4.364)%
  --(3.981,4.386)--(4.032,4.409)--(4.082,4.430)--(4.132,4.451)--(4.183,4.471)--(4.233,4.491)%
  --(4.283,4.510)--(4.334,4.529)--(4.384,4.547)--(4.434,4.565)--(4.485,4.582)--(4.535,4.599)%
  --(4.585,4.615)--(4.636,4.631)--(4.686,4.647)--(4.736,4.662)--(4.787,4.676)--(4.837,4.690)%
  --(4.887,4.704)--(4.938,4.717)--(4.988,4.730)--(5.038,4.743)--(5.089,4.755)--(5.139,4.767)%
  --(5.189,4.779)--(5.240,4.790)--(5.290,4.801)--(5.340,4.812)--(5.391,4.822)--(5.441,4.832)%
  --(5.491,4.842)--(5.542,4.852)--(5.592,4.861)--(5.642,4.870)--(5.693,4.879)--(5.743,4.887)%
  --(5.793,4.895)--(5.844,4.903)--(5.894,4.911)--(5.944,4.919)--(5.995,4.926)--(6.045,4.934);
\gpcolor{rgb color={0.000,0.000,0.000}}
\node[gp node right,font={\fontsize{10pt}{12pt}\selectfont}] at (9.537,1.948) {0.05 bound};
\gpcolor{rgb color={0.000,0.000,1.000}}
\gpsetlinetype{gp lt plot 0}
\draw[gp path] (6.597,1.948)--(7.513,1.948);
\draw[gp path] (1.012,0.985)--(1.062,1.071)--(1.113,1.155)--(1.163,1.237)--(1.213,1.318)%
  --(1.264,1.397)--(1.314,1.475)--(1.364,1.550)--(1.415,1.625)--(1.465,1.698)--(1.515,1.769)%
  --(1.566,1.839)--(1.616,1.907)--(1.666,1.974)--(1.717,2.040)--(1.767,2.104)--(1.817,2.167)%
  --(1.868,2.229)--(1.918,2.289)--(1.968,2.348)--(2.019,2.406)--(2.069,2.463)--(2.119,2.519)%
  --(2.170,2.573)--(2.220,2.627)--(2.270,2.679)--(2.321,2.730)--(2.371,2.781)--(2.421,2.830)%
  --(2.472,2.878)--(2.522,2.925)--(2.572,2.971)--(2.623,3.017)--(2.673,3.061)--(2.723,3.104)%
  --(2.774,3.147)--(2.824,3.189)--(2.874,3.229)--(2.925,3.269)--(2.975,3.309)--(3.025,3.347)%
  --(3.076,3.384)--(3.126,3.421)--(3.176,3.457)--(3.227,3.493)--(3.277,3.527)--(3.327,3.561)%
  --(3.378,3.594)--(3.428,3.627)--(3.478,3.659)--(3.529,3.690)--(3.579,3.721)--(3.629,3.751)%
  --(3.679,3.780)--(3.730,3.809)--(3.780,3.837)--(3.830,3.864)--(3.881,3.891)--(3.931,3.918)%
  --(3.981,3.944)--(4.032,3.969)--(4.082,3.994)--(4.132,4.018)--(4.183,4.042)--(4.233,4.066)%
  --(4.283,4.089)--(4.334,4.111)--(4.384,4.133)--(4.434,4.154)--(4.485,4.176)--(4.535,4.196)%
  --(4.585,4.216)--(4.636,4.236)--(4.686,4.256)--(4.736,4.275)--(4.787,4.293)--(4.837,4.312)%
  --(4.887,4.329)--(4.938,4.347)--(4.988,4.364)--(5.038,4.381)--(5.089,4.397)--(5.139,4.413)%
  --(5.189,4.429)--(5.240,4.445)--(5.290,4.460)--(5.340,4.475)--(5.391,4.489)--(5.441,4.503)%
  --(5.491,4.517)--(5.542,4.531)--(5.592,4.544)--(5.642,4.558)--(5.693,4.570)--(5.743,4.583)%
  --(5.793,4.595)--(5.844,4.607)--(5.894,4.619)--(5.944,4.631)--(5.995,4.642)--(6.045,4.653);
\gpcolor{rgb color={0.000,0.000,0.000}}
\node[gp node right,font={\fontsize{10pt}{12pt}\selectfont}] at (9.537,1.563) {$L(0.2)$};
\gpcolor{rgb color={1.000,0.000,0.000}}
\gpsetlinewidth{3.00}
\gpsetpointsize{8.00}
\gppoint{gp mark 2}{(3.338,5.186)}
\gppoint{gp mark 2}{(7.055,1.563)}
\gpcolor{rgb color={0.000,0.000,0.000}}
\node[gp node right,font={\fontsize{10pt}{12pt}\selectfont}] at (9.537,1.178) {$L(0.05)$};
\gpcolor{rgb color={0.000,0.000,1.000}}
\gppoint{gp mark 6}{(5.952,5.186)}
\gppoint{gp mark 6}{(7.055,1.178)}
%% coordinates of the plot area
\gpdefrectangularnode{gp plot 1}{\pgfpoint{1.012cm}{0.985cm}}{\pgfpoint{6.045cm}{5.186cm}}
\end{tikzpicture}
%% gnuplot variables
    \caption{Exponential prior}
    \label{fig:finite-exponential}
  \end{subfigure}
  \caption{The  mass of $L$-Lipschitz parameters, for two finite familes of Bernoulli distributions with $\Delta \in \set{0.2, 0.05}$ (thick lines) together with their respective stochastic Lipschitz bounds (thin lines) and the corresponding uniform Lipschitz constant $L$.}
  \label{fig:finite}
\end{figure}
\begin{example}[Continuation of Example~\ref{ex:finite}]
These conditions can be examined in terms of the finite family of Example~\ref{ex:finite}. %, which have a Lipschitz constant smaller than $L$.
Figure~\ref{fig:finite} demonstrates the assumptions for $\Delta = 0.2$ (red dashed lines) and $\Delta = 0.05$ (blue solid lines).

In particular, the two thick lines Figure~\ref{fig:finite-uniform} show the probability mass of $L$-Lipschitz parameters for the two families. They are both step functions, as the families are discrete.\footnote{The $\Delta = 0.2$ family only has two steps, as the Lipschitz constant is symmetric about $\param = 0.5$.} The $\times$ and $\circ$ symbols show the corresponding Lipschitz constants for the two families respectively, and we can clearly see $\Delta = 0.2$ has about half the Lipschitz constant of $\Delta = 0.05$. The thinner curves depict the highest lower bound on the probability mass defined in Assumption~\ref{ass:hoelder-measure-observations}. There we see that the higher curve is achieved by $\Delta = 0.2$.

In order to improve the lower bound, we need to modify our prior
distribution on the family members so as to place less mass on the
more sensitive parameters. The result of this operation is shown in
Figure~\ref{fig:finite-exponential}, which uses the prior
$\bel(\param) \propto \exp(-\ell(\param))$, i.e., it places
exponentially smaller weight in more sensitive parameters. This results in both lower bounds being shifted upwards, corresponding to a higher $c$ constant in Assumption~\ref{ass:hoelder-measure-observations}. Of course, this has no effect on Assumption~\ref{ass:hoelder-observations}.
\end{example}

For completeness, we now show that verifying our assumptions for a distribution of a single random variable lifts to a corresponding property for the product distribution on i.i.d. samples.
\begin{lemma}
	If $\family$ satisfies Assumption~\ref{ass:hoelder-observations} (resp. Assumption~\ref{ass:hoelder-measure-observations}) with respect to pseudo-metric $\xdistChar$ and constant $L$ (or $c$), then, for any fixed $n \in \Naturals$, the product family $\family^n$ with densities (sim. measures) $p_{\Params}^n(\set{x_i}) = \prod_{i=1}^{n}{p_{\Params}(x_i)}$ satisfies the same assumption with respect to:
\[
  \xdistChar^n( \set{x_i}, \set{y_i}) = \tsum_{i=1}^{n}{\xdist{x_i}{y_i}}
\]
and constant $L$ (or $c$). %Further, if $\set{x_i}$ and $\set{y_i}$ differ in at most $k$ items, the assumption holds with the same pseudo-metric but with constant $L \cdot k$ (or $\frac{c}{k}$) instead.
\label{lma:prodDist}
\end{lemma}

\subsubsection{Necessary Conditions}
\label{sec:necessary-conditions}
Finally, let us discuss whether the above conditions are necessary. In fact, either the first condition must be true, or a similar condition must hold on the marginals for every possible dataset pair $(x,y)$. Our second condition can be seen as a specific case of the necessary condition for the marginals, as explained below.
\begin{theorem}
  For a prior $\bel$ to be differentially private for a family $\family$, either
  \begin{align}
    \label{eq:necessary-condition}
    \sup_{\param \in \Params} \ln \frac{P_\param(x)}{P_\param(y)}& \leq L \xdist{x}{y},
                                                                   & \textrm{or} & &
    \ln \frac{\marg(y)}{\marg(x)} & \leq L \xdist{x}{y}
  \end{align}
  for all $x, y \in \CX$.
\end{theorem}
\begin{proof}
  If neither condition holds for some pair $(x,y)$ then there is $\param$ such that
$\ln \frac{P_\param(x)}{P_\param(y)} > L \xdist{x}{y}$ and $\ln \frac{\marg(y)}{\marg(x)} > L \xdist{x}{y}$.
Simply adding the two, we obtain $\ln \frac{\bel(\param \mid x)}{\bel(\param \mid y)} > 2 L \xdist{x}{y}$, and so the resulting posterior is not $L$-differentially private.
\end{proof}
In our main results, we show that the first part of the conditions, which is equivalent to our first assumption, is also sufficient. However, the second part is too weak to imply differential privacy on its own.

\subsubsection{The Choice of Metric and Sufficient Statistics}
The extent to which our assumptions hold for a particular family of distributions $\family$ depends mainly on $\xdistChar$. The choice of metric is also important for achieving differential privacy with respect to it. Let us specifically consider metrics defined in terms of a difference in statistics:
\begin{equation*}
  %\label{eq:statistics}
  \xdist{x}{y} \defn \norm{\stat(x) - \stat(y)}\enspace,
\end{equation*}
where $\stat: \CS \to \CV$ is a statistic mapping from datasets to a normed vector space.

\paragraph{Necessity for assumptions.} In that case, our assumptions imply that $\stat$ must be a \emph{sufficient} statistic, since if $\stat(x) = \stat(y)$ then $\xdist{x}{y} = 0$ and it follows that $P_\param(x) = P_\param(y)$. More generally, $\xdistChar$ must be such that if the distance between $x,y$ is zero, then their probabilities should be equal.
 We will see some examples of such statistics for conjugate distributions in the exponential family in Section~\ref{sec:examples}.
That means that we cannot use a metric which simply ignores part of the data, for example.

\paragraph{Necessity for differential privacy.} Similarly, the very definition of differential privacy (Definition~\ref{def:dp-new}) implies that $\stat$ must be a \emph{Bayes-sufficient} statistic. That means that for any $x, y$, it holds
\[
\stat(x) = \stat(y) \quad \Rightarrow \quad \bel(B \mid x)  = \bel(B \mid y)\; ,\;\;\; \forall B\in\field{\Params}\;.
\]
Note that this is a slightly weaker condition than a sufficient statistic, which is necessary for our assumptions to hold.

\if 0
\subsection{Insufficient Statistics}
Thinking about insufficient statistics brings us to an alternative notion of differential privacy.
\begin{definition}[$(\epsilon, \chi)$-Insufficient Statistic Differential Privacy ($\epsilon, \chi)$-IS-DP)]
  \begin{equation}
    \label{eq:is-dp}
    \ln \abs{\frac{P(B \mid x)}{P(B \mid y)}}
    \leq
    \epsilon \xdist{x}{y} + \chi,
  \end{equation}
  where $\epsilon, \chi > 0$ are constants.
\end{definition}
Note that the standard definition of differential privacy, by avoiding discussing the metric space induced by dataset neighbourhoods ignores this factor. It is of course true that any discussion of insufficient statistics must be
\fi

\subsection{Summary of Results}

Given the above assumptions, we show: firstly, that if we choose an informative prior $\bel$, the resulting posterior is robust in terms of KL-divergence to small changes in the data. Secondly, that the posterior distribution is differentially private. Thirdly, that this implies that sampling from the posterior can be used as part of a differentially-private mechanism. We complement these with results on how easily an adversary can distinguish two similar datasets from posterior samples. Finally, we characterise the trade-off between utility and privacy, stated here informally for ease of exposition:
\begin{claim}
If \Adv{} prefers to use the prior $\bel^\star$,
but \Bay{} uses a prior $\bel$ satisfying Assumption~\ref{ass:hoelder-observations}, and \Adv{}'s utility is bounded in $[0,1]$, the following is true for the posterior sampling mechanism with $\nsamples$ samples:
\begin{itemize}
\item The mechanism is $2 \nsamples L$-differentially private.
\item \Adv{}'s utility loss is  $O\left([1 - \bel^\star(\Params_L)] + \sqrt{1/N}\right)$ w.h.p., where $\Params_L$ is the support of $\bel$.
\end{itemize}
\end{claim}

The following sections discuss our main results in detail. We begin by proving that our assumptions result in robust posteriors, in the sense that the KL divergence between posteriors arising from similar datasets is small. Then we show that they also result in differentially private posterior distributions, and analyse the resulting posterior sampling mechanism. We conclude with some examples and a discussion of related work.

\section{Robustness of the Posterior Distribution}
\label{sec:robustness}
We now show that the above assumptions provide guarantees on the robustness of the posterior. That is, if the distance between two datasets $x, y$ is small, then so too is the distance between the two resulting posteriors, $\bel(\cdot \mid x)$ and $\bel(\cdot \mid y)$. We prove this result for the case where we measure the distance between the posteriors in terms of the well-known KL-divergence:
\begin{equation*}
  %\label{eq:kl-divergence}
  \dist{P}{Q} = \int_S \ln \frac{\dd{P}}{\dd{Q}} \dd{P} \enspace.
\end{equation*}
The following theorem shows that any distribution family $\family$ and prior $\bel$ satisfying one of our assumptions is robust, in the sense that the posterior does not change significantly with small changes to the dataset. It is notable that our mechanisms are simply tuned through the choice of prior.
\begin{theorem}
When $\bel$ is 
a prior distribution on $\Params$ and $\bel(\cdot \mid x)$ and $\bel(\cdot \mid y)$ are the respective posterior distributions for datasets $x, y \in \CS$, the following results hold:
  \begin{enumerate}
  \item Under a pseudo-metric $\rho$ and $L >  0$ satisfying Assumption~\ref{ass:hoelder-observations},
    \begin{equation}
      \dist{\bel(\cdot \mid x)}{\bel(\cdot \mid y)}
      \leq
      2L\xdist{x}{y}\enspace.
      \label{eq:kl-1}
    \end{equation}
    \label{the:kl-1}
  \item Under a pseudo-metric $\rho$ and $c > 1$ satisfying Assumption~\ref{ass:hoelder-measure-observations}
    \begin{equation}
      \dist{\bel(\cdot \mid x)}{\bel(\cdot \mid y)}
       \leq
      \constFamily
      \left(
        1 + 2 L_0 + c^{-1}
      \right)
      \xdist{x}{y}\enspace,
      \label{eq:kl-2}
    \end{equation}
    where $\constFamily$ is the ratio between the maximum and marginal likelihoods \eqref{eq:constFamily}.
    \label{the:kl-2}
  \end{enumerate}
  \label{the:kl}
\end{theorem}
%\begin{proof}See Appendix~\ref{sec:proofs}.\end{proof}

Note that the second claim bounds the KL divergence in terms of \Bay{}'s prior belief that $L$ is small, which is expressed via the constant $c$. The larger $c$ is, the less prior mass is placed in large $L$ and so the more robust inference becomes. Of course, choosing $c$ to be too large may decrease efficiency.

It is important to also discuss the constant $\constFamily$. To get a better intuition, consider the case where $\Params, \CX$ are finite. Let $\mlparam(x)$ be the maximum-likelihood estimate for $x$. Then we have that:
\begin{eqnarray}
\constFamily
&=&
\max_x \frac{P_{\mlparam(x)}(x)}
{\sum_\Params P_{\param}(x) \bel(\param)}
\ \leq\ \max_x \frac{1}{\bel(\mlparam(x))}\enspace,\label{eq:constFamily}
\end{eqnarray}
there is therefore a natural dependency on the prior mass placed on
maximum-likelihood estimators.
% CD: Maybe we need to add something about efficiency here
\section{Privacy and Utility}
\label{sec:privacy}
%In the following section, we demonstrate that if the likelihood and
%prior distributions satisfy the conditions set forth by Assumption~\ref{ass:hoelder-observations} or~\ref{ass:hoelder-measure-observations}, then the resulting posterior has desirable privacy properties, both with respect to differential privacy and with respect to distinguishability.

We next examine the differential privacy of the posterior distribution. We show in Section~\ref{sec:diff-priv-post} that this can be achieved under either of our assumptions. The result can also be interpreted as the differential privacy of a \emph{posterior sampling mechanism} for responding to queries  (described in Section~\ref{sec:post-sampl-query}), for which we prove a bound on the utility depending on the number of samples taken. Section~\ref{sec:bound-post-query} examines an alternative notion of privacy, {\em dataset distinguishability}, similar to \citet{wasserman2010statistical}. For this, we prove a bound on privacy, that also depends on the number of samples taken. Together, these exhibit a trade off between utility and privacy controlled by choosing the number of samples appropriately, in a manner described in Section~\ref{sec:trad-util-priv}.

\subsection{Differential Privacy of Posterior Distributions}
\label{sec:diff-priv-post}

We consider our generalised notion of differential privacy
for posterior distributions (Definition~\ref{def:dp-new}); and show that the type of differential privacy exhibited by the posterior depends on which assumption holds.
\begin{theorem}
  \begin{enumerate}
  \item Under a pseudo-metric $\rho$ and $L>0$ satisfying Assumption~\ref{ass:hoelder-observations}, for all $x,y \in \mathcal{\CS}$,
    $B \in \field{\Params}$:
    \begin{align*}
      \bel(B \mid x) & \leq \exp\{ 2L \xdist{x}{y} \} \bel(B \mid y)\; .
    \end{align*}
    \ie the posterior $\bel$ is
    $(2L, 0)$-differentially private under pseudo-metric $\xdistChar$.
    \label{thm:dp1}
  \item Under a pseudo-metric $\rho$ and $c > 1$ satisfying Assumption~\ref{ass:hoelder-measure-observations}, $\constFamily$ defined in~\eqref{eq:constFamily}, for
    all $x,y \in \mathcal{\CS}$, $B \in \field{\Params}$:
    \begin{align*}
      \abs{\bel(B \mid x) - \bel(B \mid y)}
      & \leq \sqrt{\frac{\constFamily}{2}\left(1 + 2 L_0 +  c^{-1} \right) \xdist{x}{y}},
    \end{align*}
    \ie the posterior $\bel$ is
    $\left(0,  O(\sqrt{\constFamily(L_0 + 1/c)})\right)$-differentially private\footnote{This holds, for example, for hamming distance as in the Beta-Binomial example presented in Lemma~\ref{lem:beta}.} under pseudo-metric $\sqrt{\xdistChar}$.
    \label{thm:dp2}
  \end{enumerate}
  \label{thm:dp}
\end{theorem}
The difference between the two bounds' form is due to the fact that while the first claim has a direct proof, the second claim arises from the KL divergence bound in Theorem~\ref{the:kl}.

Finally, we show that posterior distributions are also randomly differentially private.
\begin{corollary}Under pseudo-metric $\rho$, $c>1$ and $L\geq L_0>0$ satisfying Assumption~\ref{ass:hoelder-measure-observations}:
$$\Pr\left[\forall B \in \field{\Params}: \bel(B \mid x) \leq \exp\left\{ 2L \xdist{x}{y} \right\} \bel\left(B \mid y\right), \forall x, y\in \mathcal{\CS}\right] \geq 1-\exp(-c(L-L_{0}))\; .$$
\ie the posterior $\bel$ is $(2L, 0, \exp(-c(L-L_{0})))$-randomly differentially private~\citep{rdp} under pseudo-metric $\xdistChar$.
\end{corollary}
This is a conceptually different definition from the original RDP, as the measure over which the randomness is defined is not the data distribution, but the prior measure $\bel$.

This property of the posterior distribution directly leads to the definition of a posterior sampling mechanism which will be differentially private. This is explained in the following section.

\subsection{Posterior Sampling Mechanism}
\label{sec:post-sampl-query}

Given that we have a full posterior distribution which is differentially private, we can use it to define a private mechanism. We may allow the adversary to submit an arbitrary set of queries $\set{\query_t}$ with each $\query_t\in \Queries$. Each query warrants a response $\response_t$ in a set of possible responses $\Responses$. The adversary is allowed to condition the queries on our previous responses.

We extend our original approach~\citep{alt:robust} to take some
utility function $u$ into account, which scores preferences of
responses given a query. The algorithm requires a prior $\bel$ to be
defined on a family $\family$ of probability distributions, whose
members do not necessarily generate i.i.d. observations. They could be
Markov chains for example.  The first step is to simply draw a number
of samples from the posterior, as in the original approach
(Algorithm~\ref{alg:posterior-sampling-query}). After the algorithm
calculates the posterior distribution $\bel(\cdot \mid x)$,
$\nsamples$ parameter samples are drawn from it, producing a parameter
set $\hat{\Params}$. Thereafter, responses depend only on the utility
function and the sample $\hat{\Params}$, and we do not draw new
samples after every query. This allows us to work with a fixed privacy
budget.

\begin{algorithm}[htb]
  \begin{algorithmic}[1]
    \STATE \textbf{input} prior $\bel$, data $x \in \CS$
    \STATE Calculate posterior $\bel(\param \mid x)$.
    \FOR {$k = 1, \ldots, \nsamples$}
    \STATE Sample $\param^{(k)} \sim \bel(\param \mid \data)$.
    \ENDFOR
    \STATE \textbf{return} $\hat{\Params} = \cset{\param^{(k)}}{k=1, \ldots, \nsamples}$.
  \end{algorithmic}
  \caption{BAPS: Bayesian Posterior Sampling}
  \label{alg:bayesian-posterior-sampling}
\end{algorithm}

\begin{corollary}
	Algorithm~\ref{alg:bayesian-posterior-sampling} is differentially private under the conditions of Theorem~\ref{thm:dp}, namely:
	\begin{enumerate}
  \item Under a pseudo-metric $\rho$ and $L>0$ satisfying Assumption~\ref{ass:hoelder-observations}, the algorithm is $(2L, 0)$-differentially private under pseudo-metric $\xdistChar$; or
  \item Under a pseudo-metric $\rho$ and $c > 1$ satisfying Assumption~\ref{ass:hoelder-measure-observations}, $\constFamily$ defined in~\eqref{eq:constFamily}, the algorithm is 
    $\left(0,  O(\sqrt{\constFamily (L_0 + 1/c)})\right)$-differentially private under pseudo-metric $\sqrt{\xdistChar}$.
    \end{enumerate}
\end{corollary}
\begin{proof}
	This follows directly from Theorems~\ref{thm:dp} and \ref{thm:composition} (composition), as the algorithm samples from the posterior distribution, which is differentially private.
\end{proof}

\paragraph{Utility and optimal responses.}
We assume the collection of a set of utility functions $\Utils = \cset{\util_\param}{\param \in \Params}$, such that the optimal response for a given parameter $\param$ is the one maximising a utility function $\util_\param : \Queries \times \Responses \to [0,1]$. If we know the true parameter $\param$, then we should respond to any query $\query$ with $\response \in \argmax_\response \util_\param(\query, \response)$. However, since $\param$ is unknown, we must select a method for conveying the required information. In a Bayesian setting, there are three main approaches we could employ. The standard methodology is to maximise \emph{expected utility} with respect to the posterior. This corresponds to marginalising out $\param$, and responding with:
\begin{eqnarray*}
	\response_t &\in& \argmax_\response \int_\Params \util_\param(\query_t, \response) \dd{\bel}(\param \mid x)\enspace.
  %\label{eq:expected-utility}
\end{eqnarray*}
The second is to use the {\em maximum a posteriori} value of $\param$. The final, which we employ here, is to use sampling; \ie to reply to each query using parameters sampled from the posterior. This allows us to reply to arbitrary queries without compromising privacy, since the most information an adversary could obtain is the set of sampled parameters. By adjusting the number of samples used, we can easily trade off between privacy and utility.

After this we respond to a series of queries. For the $t$-th received query $\query_t$, the algorithm returns the optimal response over the sampled parameter set $\hat{\Params}$, in the manner shown in Algorithm~\ref{alg:posterior-sampling-query}. Since we allow arbitrary queries, the third party could simply ask for $\hat{\Params}$ with a suitable choice of the utility function.
Then if $\util$ is bounded, it is easy to show that the loss due to sampling is bounded.
\begin{algorithm}[htb]
  \begin{algorithmic}[1]
    \STATE \textbf{input} Parameter sample $\hat{\Params}$.
    \FOR {$t=1, \ldots$}
    \STATE Observe query $\query_t \in \Queries$, perhaps depending on $\response_1,\ldots,\response_{t-1}$ and $\query_1,\ldots,\query_{t-1}$.
    \STATE \textbf{return} $\response_t \in \argmax_\response \sum_{\param \in \hat{\Params}} \util_{\param}(\query_t, \response)$
    \ENDFOR
  \end{algorithmic}
  \caption{PSAQR: Posterior Sample Query Response}
  \label{alg:posterior-sampling-query}
\end{algorithm}

\begin{lemma}
	The returned responses of the PSAQR mechanism have a utility which is within \linebreak[4] $\BigO{\sqrt{\ln(1/\delta)/\nsamples}}$ of the optimal value with probability at least $1 - \delta$ for any $\delta > 0$.
\label{lem:pac-utility}
\end{lemma}
Now that we have demonstrated bounds on the utility for the algorithm above, we turn to the issue of how utility and privacy can be optimally tuned. First, we try and quantify the amount of samples an adversary needs to distinguish two datasets.

\subsection{Distinguishability of Datasets}
\label{sec:bound-post-query}
In this section, we wish to relate the size of the sample $\hat{\Theta}$ to the amount of information about $x$ that can be obtained by the adversary \Adv{}. More precisely, we need to bound how well $\Adv{}$ can distinguish $x$ from all alternative datasets $y$. Within the posterior sampling query model, \Adv{} has to decide whether \Bay{}'s posterior is $\bel(\cdot \mid x)$ or $\bel(\cdot \mid y)$. However, he can only do so within some neighbourhood $\epsilon$ of the original data. In this section, we bound \Adv{}'s error in determining the posterior in terms of the number of samples used. This is analogous to the dataset-size bounds on queries in interactive models of differential privacy~\citep{Dwork06}, as well as the point of view of privacy as hypothesis testing~\citep{oh2013composition,wasserman2010statistical} where an adversary wishes to distinguish the dataset from two alternatives.

For this section, we consider a utility function whose optimal response is $\hat{\Theta}$.  This corresponds to the most powerful query possible under the model shown in Algorithm~\ref{alg:posterior-sampling-query}. Then, the adversary needs only to construct the empirical distribution to approximate the posterior up to some sample error. By bounds on the KL divergence between the empirical and actual distributions we can bound his power in terms of how many samples he needs in order to distinguish between $x$ and $y$.

Due to the sampling model, we first require a finite sample bound on the quality of the empirical distribution. The adversary could attempt to distinguish different posteriors by forming the empirical distribution on any sub-algebra $\algebra$.
\begin{lemma}
  For any $\delta \in (0,1)$, let $\partition$ be a finite partition of the sample space $\CS$, of size $m \leq \log_2 \sqrt{1/\delta}$, generating the $\sigma$-algebra $\algebra = \sigma(\partition)$. Let $x_1, \ldots, x_n \sim P$ be i.i.d. samples from a probability measure $P$ on $\CS$, let $P_{|\algebra}$ be the restriction of $P$ on $\algebra$ and let $\hat{P}^n_{|\algebra}$ be the empirical measure on $\algebra$. Then, with probability at least $1 - \delta$:
  \begin{equation}
    \label{eq:l1-empirical-error}
    \onenorm{\hat{P}^n_{|\algebra} - P_{|\algebra}} \leq \sqrt{\frac{3}{n} \ln \frac{1}{\delta}}\; . % ,
   %\qquad \forall \delta \in (0,1).
  \end{equation}
  \label{lem:l1-empirical-error}
\end{lemma}

We can combine this bound on the adversary's estimation error with Theorem~\ref{the:kl}'s bound on the KL divergence between posteriors resulting from similar data to obtain a measure of how fine a distinction between datasets the adversary can make after a finite number of draws from the posterior:
\begin{theorem}
  Under Assumption~\ref{ass:hoelder-observations}, the adversary can distinguish between data $x, y$ with probability $1 - \delta$ if:
  \begin{equation*}
     \xdist{x}{y} \geq \frac{3}{4Ln} \ln \frac{1}{\delta}\enspace.
  \end{equation*}
  Under Assumption~\ref{ass:hoelder-measure-observations}, this becomes:
  \begin{equation*}
    %\label{eq:adversary-sample-complexity-2}
    \xdist{x}{y} \geq \frac{6}{n\left(\constFamily c^{-1} + \ln \constFamily\right)} \ln \frac{1}{\delta}\enspace.
  \end{equation*}
  \label{the:adversary-sample-complexity}
\end{theorem}
%\begin{proof}See Appendix~\ref{sec:proofs}.\end{proof}

Consequently, either smoother likelihoods (\ie decreasing $L$), or a larger concentration on smoother likelihoods (\ie increasing $c$), increases the effort required by the adversary and reduces the sensitivity of the posterior. Note that, unlike the results obtained for differential privacy of the posterior sampling mechanism, these results have the same algebraic form under both assumptions.

\subsection{Trading off Utility and Privacy}
\label{sec:trad-util-priv}
By construction, in our setting there are three ways with which to tune privacy. The first is the choice of family; the second is the choice of prior; and the third is how many samples $\nsamples$ to draw. The choice of family is usually fixed due to other considerations. However, we have the choice of either tuning the prior, so that we can satisfy our assumptions with some suitable constants $L$ or $c$, or by tuning the number of samples $\nsamples$ in the posterior sampling framework.

The following lemma bounds the regret we suffer in terms of utility when the private posterior we use is $\bel$, in the case where the posterior we would like to use (assuming no privacy constraints) was $\bel^\star$.
\begin{lemma}
  If our utility is bounded in $[0,1]$, the private posterior we use is $\bel$, while the ideal posterior is $\bel^\star$, then the regret suffered is bounded by $2 \|\bel - \bel^\star\|_1$.
  \label{lem:utility-bound}
\end{lemma}
Finally, consider the case where \Bay{}, being a true Bayesian, is convinced that $\bel^\star$ is the correct prior distribution to use, but needs to use the prior $\bel$ in order to achieve privacy. The following theorem bounds the expected KL divergence between the two resulting posteriors.
\begin{lemma}
  If $\forall \param \in \Params$, $|\ln \bel^\star(\param) / \bel(\param)| \leq \eta$ then the expected KL divergence is
\begin{equation*}
  \E_{x \sim \marg^\star} D(\bel^\star(\cdot \mid x) \| \bel(\cdot \mid x)) \leq 2 \eta\; ,
\end{equation*}
where $\marg^\star$ is the $\bel^\star$ marginal distribution.
\label{lem:expected-kl-divergence}
\end{lemma}
We can now combine Lemmas~\ref{lem:pac-utility} and \ref{lem:utility-bound} with Lemma~\ref{lem:expected-kl-divergence}, to obtain the following result:
\begin{corollary}
	If \Adv{} has a preferred prior $\bel^\star$, while the private prior used by \Bay{} is $\bel$ and it satisfies the conditions of Lemma~\ref{lem:expected-kl-divergence}, then the loss of \Adv{} in terms of the $\bel^\star$-expected utility is $\BigO{\eta + \sqrt{\ln (1/\delta)/\nsamples}}$, with probability at least $1 - \delta$.
\end{corollary}
Consequently, if \Adv{} believes the correct prior should be $\bel^\star$, he can use the private posterior sample to make decisions, incurring a small loss. Finally, we already showed that \Adv{} cannot distinguish between data that are closer than $\BigO{1/\nsamples}$ with high probability. Hence, in this setting we can tune $\nsamples$ to trade off utility and privacy.

The following theorem characterises the link between the choice of prior, the number of samples, privacy and utility directly. This connects several of our results in one place.
\begin{theorem}
  	If, instead of using a non-private prior $\bel^\star$, we use a prior $\bel$ restricted on $\Params_L$ (such that it satisfies Assumption~\ref{ass:hoelder-observations} with constant $L$) and generate $\nsamples$ samples from the posterior, then (a) the sample is $2 L \nsamples$-differentially private and (b) the loss of \Adv{} in terms of the $\bel^\star$-expected utility is $\BigO{[1 - \bel^\star(\Params_L)] + \sqrt{\ln (1/\delta)/\nsamples}}$, with probability at least $1 - \delta$ for any $\delta > 0$.
\end{theorem}
\begin{proof}
For (a) note that due to composition, $\nsamples$ repetitions give $2L \nsamples$-differential privacy.
For (b), let $\Params_L$ be the support of $\bel$. Then, because $\bel$ is the restriction of $\bel^\star$ on $\Params_L$, it holds that:
\begin{align*}
  %\label{eq:support-bound}
  \onenorm{\bel - \bel^\star}
  & =
    \bel( \Params_L)
    -
    \bel^\star( \Params_L)
    +
    \bel^\star(\Params \setminus \Params_L)
    -
    \bel(\Params \setminus \Params_L)
  \\
  & =
    2[1 - \bel^\star( \Params_L)]\enspace.
\end{align*}
We now just need to couple this with Lemmas~\ref{lem:utility-bound} and \ref{lem:pac-utility} to directly obtain the stated bound on the utility.
\end{proof}
In practice, our choice of $\bel$ gives us a base amount of privacy that depends only on $L$. By keeping $\bel$ fixed and increasing $N$, we can easily trade off privacy and utility.

Finally, we should note that the adversary could choose any arbitrary estimator $\psi$ to guess $x$. Section~\ref{app:le-cam} below describes how to apply Le Cam's method to obtain matching lower bounds in this case, by defining \emph{dataset estimators} as a model for the adversary.

\subsection{Lower Bounds}
\label{app:le-cam}
It is possible to apply standard minimax theory to obtain lower bounds on the rate of convergence of the adversary's estimate to the true data. In order to do so, we can for example apply the method due to \citet{lecam1973convergence}, which places lower bounds on the expected distance between an estimator and the true parameter. In order to apply it in our case, we simply replace the parameter space with the dataset space.

Le Cam's method assumes the existence of a family of probability measures indexed by some parameter, with the parameter space being equipped with a pseudo-metric. In our setting, we use Le Cam's method in a slightly unorthodox, but very natural manner. Define the family of probability measures on $\Params$ to be:
\begin{equation*}
  %\label{eq:posterior-family}
  \Bels \defn \cset{\bel(\cdot \mid x)}{x \in \CS},
\end{equation*}
the family of posterior measures in the parameter space, for a specific prior $\bel$. Consequently, now $\CS$ plays the role of the parameter space,  while $\xdistChar$ is used as the pseudo-metric. The original family $\family$ plays no further role in this construction, other than a way to specify the posterior distributions from the prior.

Now let $\psi$ be an arbitrary estimator of the unknown data $x$. As in \citep{lecam1973convergence}, we extend $\xdistChar$ to subsets of $\CS$ via
\begin{align*}
\xdist{A}{B} &\defn \inf\cset{\xdist{x}{y}}{x \in A, y \in B}\; , \; \; \;
A, B \subset \CS\; .
\end{align*}

Now we can re-state the following well-known lemma for our specific setting.
\begin{lemma}[Le Cam's method]
  Let $\psi$ be an estimator of $x$ on $\Bels$ taking values in the metric space $(\CS, \xdistChar)$. Suppose that there are well-separated subsets $\CS_1, \CS_2$ such that $\xdist{\CS_1}{\CS_2} \geq 2 \delta$. Suppose also that $\Bels_1, \Bels_2$ are subsets of $\Bels$ such that $x \in \CS_i$ for $\bel(\cdot \mid x) \in \Bels_i$.
Then:
\begin{align*}
  \sup_{x \in \CS} \E_\bel (\xdist{\psi}{x} \mid x) \geq \delta \sup_{\bel_i \in \textrm{co}(\Bels_i)} \|\bel_1 \wedge \bel_2\|\enspace.
\end{align*}
\end{lemma}
This lemma has an interesting interpretation in our case. The quantity
\[
\E_\bel (\xdist{\psi}{x} \mid x)
=
\int_\Theta \xdist{\psi(\theta)}{x} \dd{\bel}(\theta \mid x)\enspace,
\]
is the expected distance between the real data $x$ and the guessed data $\psi(\theta)$ when $\theta$ is drawn from the posterior distribution. Consequently, it is possible to apply this method directly to obtain results for specific families of posteriors. These would of course be dependent on the family, the prior and the metric. While we shall not engage in this exercise, we point the interested reader to \citep{yu1997assouad}, which provides two simple examples with minimax rates of $O(n^{-4/9})$ and $O(n^{-4/5})$.

\section{Examples Satisfying our Assumptions}
\label{sec:examples}
In what follows we study, for different choices of likelihood and corresponding conjugate prior, what constraints can be placed on the prior's concentration to guarantee a desired level of privacy. These case studies closely follow the pattern in differential privacy research where the main theorem for a new mechanism is a set of sufficient conditions on (\eg Laplace) noise levels to be introduced to a response in order to guarantee a level $\epsilon$ of $\epsilon$-differential privacy.

For exponential families, we have the canonical form
$p_{\param}(x) = h(x) \exp\left\{
  \eta_{\param}^{\top} \stat(x) - A(\eta_{\param})\right\},$
where $h(x)$ is the base measure, $\eta_{\param}$ is the distribution's natural parameter
corresponding to $\param$, $\stat(x)$ is the distribution's
sufficient statistic, and $A(\eta_{\param})$ is its log-partition
function. For distributions in this family, under the absolute
log-ratio distance, the family of parameters $\Params_L$ of
Assumption~\ref{ass:hoelder-measure-observations} must satisfy, for all $x, y \in \CS$:
$\left|\ln\frac{h(x)}{h(y)} + \eta_{\param}^{\top}\left(\stat(x)-\stat(y)\right) \right| \le L \xdist{x}{y}.$
If the left-hand side has an amenable form, then we can quantify the
set $\Params_L$ for which this requirement holds.  Particularly, for distributions where
$h(x)$ is constant and $\stat(x)$ is scalar (\eg Bernoulli, exponential,
and Laplace), this requirement simplifies to
$\frac{\left|\stat(x)-\stat(y)\right|}{\xdist{x}{y}} \le
\frac{L}{\eta_{\param}}$. One can then find the supremum of the
left-hand side independent from $\param$, yielding a simple formula
for the feasible $L$ for any $\param$. For each example, a detailed proof can be found in Appendix~\ref{sec:example-proofs}.
Note that in the following examples, we are making the conventional assumption in machine learning that data are bounded ($||x||\leq B$).
Also we use $\xi(\theta)\mathbbm{1}_{[c_{1},c_{2}]}$ to denote the trimmed density function obtained by setting the density outside $[c_1, c_2]$ to zero and renormalising the density.
%\comment{CD: Clarify. Is this really true? This would make for a weird density. Are you sure we are not just setting everything beyond $c_1, c_2$ to zero and rescaling the density?}

%Suppose that we have $n$ data points drawn from an exponential.
%distribution with a rate parameter $\param > 0$, using an exponential
%prior with rate parameter $c>0$, where the exponential distribution
%is given by $p_{\param}(x) = \param \exp (-\param x)$ for $x \ge 0$.
%Because this prior is conjugate to the likelihood, it has a closed
%form as a gamma distribution: $\GammaDist(n+1, c + \sum{x_i})$.

We begin with a few simple examples for single observations, that are
nevertheless illustrative.
\begin{lemma}[Exponential-Exponential conjugate prior]
The exponential distribution $\Exp(x;\theta)$ with a trimmed exponential conjugate prior $\theta \sim \Exp(\theta;\lambda)\mathbbm{1}_{[c_{1},c_{2}]}$, $\lambda > 0$, satisfies Assumption~\ref{ass:hoelder-measure-observations} with parameter $c = \lambda$, $L_{0}=c_{1}$, $\constFamily=c_{2}/\min\left\{c_{1}e^{-c_{1}B},c_{2}e^{-c_{2}B}\right\}$ and metric $\rho(x, y) = |x - y|$.
\label{lem:exponential}
\end{lemma}
Consequently, the trimmed-exponential prior results in a posterior sampling mechanism that is $(0,\delta)$-DP under $\xdistChar$, with $\delta = \sqrt{\frac{1}{2} \constFamily(1 + 2 c_1 + 1/\lambda)}$. It is also $(0,\delta)$-DP under the classical definition if $x, y \in [0,1]$.

\begin{lemma}[Laplace-Exponential conjugate prior]
The distribution $Laplace(x;s,\mu)$ with a trimmed exponential conjugate prior $1/s=\theta \sim Exp(\theta;\lambda)\mathbbm{1}_{[c_{1},c_{2}]}$,  $\mu \in \Reals$, $s \ge 1/L$, $\lambda > 0$ satisfies Assumption~\ref{ass:hoelder-measure-observations} with parameters $c=\lambda$, $L_{0}=c_{1}$,
$$\constFamily =\begin{cases} \frac{c_{2}}{2\min\left\{\frac{1}{2c_{2}},\frac{1}{2c_{1}}\exp\left(\frac{-B-\mu}{c_{1}}\right)\right\}}\enspace, & x < \mu \\
                    \frac {c_{2}}{2\min\left\{\frac{1}{2c_{2}},\frac{1}{2c_{1}}\exp\left(\frac{\mu-B}{c_{1}}\right)\right\}}\enspace, &   x\geq \mu
       \end{cases}\enspace,$$
and metric $\rho(x, y) = |x - y|$.
\label{lem:laplace}
\end{lemma}
It should come as no surprise that the same type of $(0,\delta)$-privacy is achieved for the Laplace distribution with a trimmed exponential prior. Now we move on to an example from which  we draw multiple samples.

\begin{lemma}[Beta-Binomial conjugate prior]
The Binomial distribution $\Binomial(\theta, n)$, with prior $\theta \sim \Beta (\alpha, \beta)$, $\alpha = \beta > 1$ satisfies Assumption~\ref{ass:hoelder-measure-observations} for $L_{0}=\ln{n}$,
$c = 2^{-2\alpha+1}/B(\alpha)$, where $B(\alpha)$ denotes the beta function with parameters $\alpha=\beta$,
$$\constFamily=B(\alpha)/B\left(\frac{n+2\alpha-1}{2},\frac{n+2\alpha+1}{2}\right)$$
and metric $\rho(x, y) = \|x - y\|_1$, where $x, y \in \{0,1\}^n$. 
\label{lem:beta}
\end{lemma}
This is an example of a conjugate prior pair that is $(0,\delta)$-DP without trimming the prior, with $\delta=\sqrt{\frac{1}{2} \constFamily (1 + 2 \ln n + 2^{2\alpha-1} B(\alpha))}$. Unfortunately, $\delta$ is increasing with $n$, and as~\cite{posterior-thesis} shows, this result is essentially unimprovable with direct posterior sampling unless the prior is trimmed.

We next present two results on normal distributions.
\begin{lemma}[Normal distribution with known mean and unknown variance]
	The normal distribution $N(x;\mu, \sigma^2)$ with a trimmed exponential prior $1/\sigma^2=\theta\sim \Exp(\theta;\lambda)\mathbbm{1}_{[c_{1},c_{2}]}$ satisfies Assumption~\ref{ass:hoelder-measure-observations} with parameter $c=\frac{2\lambda}{\max\set{|\mu|, 1}}$, $L_{0}=\frac{c_{1}\max\set{|\mu|, 1}}{2}$, $$\constFamily=\min\left\{\sqrt{c_{2}/c_{1}}\exp\left(\frac{c_{1}c^{2}_{2}}{2}\right), \exp\left(\frac{c^{3}_{2}}{2}\right)\right\}$$ and metric $\rho(x, y) = \abs{x^2 - y^2} + 2 \abs{x-y}$.
\label{lem:normal}
\end{lemma}
This example is interesting, because privacy is achieved under a rather unusual metric. However, note that the posterior is classically $(0, 3 \delta)$-DP for data in $[0,1]$.

Unbounded observation spaces are generally a problem for privacy, even for finite parameter spaces, generally because likelihoods become vanishingly small, thus making log likelihood ratios arbitrarily large. However, the following two examples circumvent this problem.
In the first example, we consider a general multivariate extension of Lemma~\ref{lem:normal}.
In the second we consider the case of discrete Bayesian networks,
where privacy depends on the network connectivity and the probability of rare
events---we have also considered posterior sampling of networks under complementary
conditions, and output perturbation applied to posterior updates, in recent work~\citep{AAAI16}.
In these examples, data is usually not i.i.d. (depending on the choice of network or covariance matrix) and the observation space is not a product space. 

\begin{lemma}[Multivariate normal distribution]
The multivariate normal distribution $N(x;\mu, A^{-1})$ satisfies our Assumption~\ref{ass:hoelder-observations} with
$L=\frac{1}{2}(\sum^{n}_{i=1}\lambda^{2}_{i})^{\frac{1}{2}}\max\{1, ||\mu||_{2}\}$ under metric $\rho(x,y)=||xx^\top-yy^\top||_{F}+2||x-y||_{2}$.
%When covariance and mean are both unknown, it satisfies with $L=\frac{1}{2}(\sum^{n}_{i=1}\lambda^{2}_{i})^{\frac{1}{2}}\max\{1, ||\mu||_{2}\}$ under metric $\rho(x,y)=||xx'-yy'||_{F}+2||x-y||_{F}$.
When $\mu=0$, Assumption~\ref{ass:hoelder-observations} is satisfied with $L=\frac{1}{2}(\sum^{n}_{i=1}\lambda^{2}_{i})^{\frac{1}{2}}$ under metric $\rho(x,y)=||(xx^\top-yy^\top)||_{F}$.
\label{lem:multinormal}
\end{lemma}
%This is the more general case for multivariate normal distributions.
Once more, we achieved $(\epsilon, 0)$-DP under our metric, which implies a $(3 \epsilon, 0)$ classical DP for bounded data.

\begin{lemma}[Discrete Bayesian networks]
  Consider a family of discrete Bayesian networks on $K$ variables,
  $\family = \cset{P_\param}{\param \in \Params}$.  More specifically,
  each member $P_\param$, is a distribution on a finite space
  $\CS = \prod_{k=1}^K \CS_k$ and we write $P_\param(x)$ for the
  probability of any outcome $x = (x_1, \ldots, x_K)$ in $\CS$.  Let
  $\varepsilon \defn \min_{\param, x_{k}, x_{\Parents{k}}} P_\param(x_k
    \mid x_{\Parents{k}})$,
    be the smallest conditional probability in the graph, where
    $\Parents{k}$ are the parents of node $k$.

    Our observations can be independent samples
    $\cset{x^t}{t \in [T]}$ of dependent variables
    $x^t_1, \ldots, x^t_k$. Define the connectivity vector
    $v \in \Naturals^K$ such that $v_k = 1 + \degree{k}$, where
    $\degree{k}$ is the out-degree of node $K$. We now define the
    distance between two datasets $x,y$ to be
  \[
  \xdist{x}{y} \defn v^\top \delta(x, y),
  \qquad
  \delta_k(x,y) \defn \sum_{t=1}^T \ind{x_{k,t} \neq y_{k,t}}.
  \]
  Then Assumption~\ref{ass:hoelder-observations} is satisfied with
  $L = \ln 1 / \varepsilon$.
\label{lem:dbn}
\end{lemma}
Consequently, discrete Bayesian networks, endowed with any prior on
the family given in the above example, are $(2 \ln 1/\varepsilon, 0)$-DP under
$\xdistChar$. This also implies that they are
$2 \|v\|_\infty \ln 1/\varepsilon$-DP under the classical definition.

A simple application of this example is to data drawn from a Markov model on a finite state space. In particular, consider a time-homogeneous family of transition matrices $\theta_{i,j} \defn P_\param(x_{t+1} = i \mid x_t = j)$. Then a prior consisting of product of truncated Dirichlet distributions that bound all multinomial probabilities above $\varepsilon$ satisfies our assumptions and results in a $4 \ln 1/\varepsilon$-DP mechanism.

The above examples demonstrate that our assumptions are reasonable. In fact, for several of them we recover standard choices of prior distributions. However, for the privacy guarantees to be reasonable, it is best to restrict the prior to a set of parameters that is not very sensitive.

%Combined then, with the following theorem, we show that differential privacy of the Bayesian network's posterior distribution can be achieved. %when the conditional likelihood at every node of the Bayesian network satisfy Assumption~\ref{ass:hoelder-observations} or Assumption~\ref{ass:hoelder-measure-observations}.

\section{Discussion}
\label{sec:conclusion}
We have presented a unifying framework for private and secure inference in a Bayesian setting. Under concentration conditions on the prior, we have shown that Bayesian inference is both robust and private. Firstly, we prove that similar datasets result in posterior distributions with small KL divergence. Secondly, we establish that the posterior is differentially private.  This allows us to use a general posterior sampling mechanism for responding to queries, where privacy and utility are easy to trade off by adjusting the number of samples taken.

Owing to the fact that no additional machinery is required, this framework may  serve as a fundamental building block for more sophisticated, private Bayesian inference. As an additional step towards this goal, we have demonstrated the application of our framework to deriving analytical expressions for well-known distribution families, and for discrete Bayesian networks.  Finally, we bounded the amount of effort required of an attacker to breach privacy when observing samples from the posterior. This serves as a principled guide for how much access can be granted to querying the posterior, while still guaranteeing privacy.

\emph{Conversion of our results to the neighbourhood formulation.} We state most of our results on specific models using a distance based on a sufficient statistic. Hence, to convert these to standard differential privacy, we only need to bound the $\rho$-distance of any neighbouring datasets. A good example are DBNs, where the case $\rho(x,y)=1$ corresponds exactly to that of one record changing in a databse.

\emph{Practical application of our results.} In general, it is hard to verify whether an existing model family will satisfy DP, because it implies checking whether the log-likelihood function is Lipschitz. Some parametric conjugate families, like the ones we examined in the examples, are amenable to analytic treatment. In practice, though, this might not be possible. It is for this reason that we propose to use rejection sampling in order to sample from the truncated posterior distribution. In particular, it is possible to resample from the posterior distribution, until a sample within the allowed interval of parameters is obtained. This is an approach we recently used in an  application paper successfully~\citep{AAAI16}.

\subsection{Related Work}
\label{sec:related-work}
%Differential privacy, first proposed by \citet{Dwork06}, has achieved
%prominence in the theory of computer science, databases, and more
%recently learning communities. Its success is largely due to the
%semantic guarantee of privacy it formalises. Differential privacy is %normally defined with respect to a randomised mechanism for responding to queries. Informally, a  mechanism preserves differential privacy if perturbing one training instance results in a small change to the mechanism's response distribution; a formal definition is deferred to Section~\ref{sec:setting}.

In the past, little research in differential privacy focused on the Bayesian paradigm, with~\citet{alt:robust} being the first to establish conditions for differentially-private Bayesian inference. Nevertheless, our paper has many interesting links with both previous and follow up work, with respect to differential privacy, robustness and Bayesian inference, which we outline below. First, we discuss relations to other mechanisms achieving differential privacy and theoretical works about differential privacy; secondly, we discuss related work on the connection between robustness and privacy; and we conclude the related work section with a discussion of previous versions of this paper and follow-up work.

\subsubsection{Differential Privacy}
\label{sec:other-mechanisms}

In our paper, we employ a Bayesian framework whereby optimal responses are characterised by the fact that they maximise expected utility.
In Bayesian statistical decision
theory~\citep{Berger-SDT-1985,Bickel-Doksum-MS-2001,Degroot:OptimalStatisticalDecisions},
learning is cast as a statistical inference problem and
decision-theoretic criteria are used as a basis for assessing,
selecting and designing procedures. In particular, for a given utility
function, the Bayes-optimal procedure maximises the expected utility
under the posterior distribution.

In our setting, however, decisions using the data are not taken by the
statistician \Bay{}. Instead, \Adv{} provides a utility function, and
trusts \Bay{} to give him responses to queries that maximise expected
utility. However \Bay{} must also balance the need for privacy of the data provider, which results in some utility loss for \Adv{}. This is naturally captured by the difference in utility by making the decision private. This idea had already been explored in the exponential mechanism by~\citet{MechDesign}, which connected differential privacy to mechanism design.

The exponential mechanism can be seen as a generalisation of the
\emph{Laplace mechanism}, which adds Laplace noise to released
statistics~\citep{Dwork06}. The exponential mechanism releases a
response with probability exponential in a utility function describing
the usefulness of each response, with the best response having maximal
utility. An alternate approach, employed for privatising regularised
empirical-risk minimisation~\citep{dpERM}, is to alter the inferential
procedure itself, in that case by adding a random term to the primal
objective. We view our posterior sampling mechanism as a Bayesian
counterpart. Further results on the accuracy of the exponential
mechanism with respect to the Kolmogorov-Smirnov distance are given
in~\citep{wasserman2010statistical}, which introduced the concept of
privacy as hypothesis testing where an adversary wishes to
distinguish two datasets. This is similar to our notion of dataset distinguishability.

\emph{Learning from private data.} In a different direction, \citet{duchi:local-privacy} provided information-theoretic bounds for private learning. This essentially represents the protocol for interacting with an adversary as an arbitrary conditional distribution, rather than restricting it to specific mechanisms or models. In this way, they obtain fundamental bounds on rates of convergence from differentially-private views of data.

\emph{Bayesian inference and privacy.} Other work at the intersection of privacy and Bayesian inference includes that of \citet{ProbInfDP} who applied Bayesian inference to improve the utility of
differentially-private releases by computing posteriors in a noisy measurement model.
In a similar vein, \citet{xiao2012bayesian} used Bayesian credible intervals to respond to queries with as high utility as possible, subject to a privacy budget. In the PAC-Bayesian setting, \citet{mir2012differentially} showed that the Gibbs estimator~\citep{MechDesign} is differentially private. While their algorithm corresponds to a posterior sampling mechanism, it is a posterior found by minimising risk bounds; by contrast, our results are purely Bayesian and come from conditions on the prior. It is also worthwhile noting that our Assumption 1 can in some cases be made equivalent to the definition of Pufferfish privacy~\citep{kifer2014pufferfish}, a privacy concept with Bayesian semantics. Thus, our results imply that in some cases Pufferfish privacy also results in differential privacy.
Finally, independently to our preliminary work~\citep{alt:robust}, \citet{smola:icml} later proved differential privacy results for Gaussian processes under similar assumptions.

\subsubsection{Robustness and Privacy}
\label{sec:robustness-privacy}
\citet{dwork2009differential} made the first connection between (frequentist) robust statistics and differential privacy, developing mechanisms for the interquartile, median and $B$-robust regression. While robust statistics are designed to operate near an ideal distribution, they
can have prohibitively high global, worst-case sensitivity. In this case privacy was still achieved by performing a differentially-private test on local sensitivity before release~\citep{dpStats}. In later work, \citet{dwork2015reusable} show that differentially-private views of the data result in good generalisation abilities. We discuss this more extensively in Section~\ref{sec:follow-up}.

In a similar vein~\citet{chaudhuri2012convergence} drew a quantitative connection between robust statistics and differential privacy by providing finite-sample convergence rates for differentially-private plug-in statistical estimators in terms of the \emph{gross error sensitivity}, a common measure of robustness. These bounds can be seen as complementary to ours because our Bayesian estimators do not have private views of the data but use a suitably-defined prior instead.

Smoothness of the learning map, achieved here for Bayesian inference by appropriate concentration of the prior, is related to \emph{algorithmic stability} which is used in statistical learning theory to establish error rates~\citep{Bousquet02}. \citet{privateSVM} used $\gamma$-uniform stability to calibrate the level of noise when using the Laplace mechanism to achieve differential privacy for the SVM. \citet{InfiniteHall} extended this technique to adding Gaussian process noise for differentially private release of infinite-dimensional functions lying in an RKHS.

 In the Bayesian setting, robustness is typically handled through
maximin policies. This is done by assuming that the prior distribution is
selected arbitrarily by nature.  In the field of robust statistics,
the minimax asymptotic bias of a procedure incurred within an
$\varepsilon$-contamination neighbourhood is used as a robustness
criterion giving rise to the notions of a procedure's \emph{influence
  function} and \emph{breakdown point} to characterise
robustness~\citep{Hampel-Ronchetti-RS-1986,Huber-RS-1981}.  In a
Bayesian context, robustness appears in several guises including
minimax risk, robustness of the posterior within
$\varepsilon$-contamination neighbourhoods, and robust
priors~\citep{Berger-SDT-1985}.  In this context
\citet{Gruenwald-Dawid-AS-2004} demonstrated the link between
robustness in terms of the minimax expected score of the likelihood
function and the (generalised) maximum entropy principle, whereby
nature is allowed to select a worst-case prior.

\subsubsection{Previous Versions and Follow Up Work}
\label{sec:follow-up}
Finally, we note that preliminary versions of this work appeared on arXiv~\citep{arxiv:robust} and ALT~\citep{alt:robust}. This version corrects technical issues with one proof, which affected the leading constants. We also replaced the original mechanism with one taking a fixed sample, which allows us to maintain a fixed privacy budget for an arbitrary number of queries. We make a novel use of Le Cam's method to prove lower bounds on indistinguishability, and we complement our original bounds with bounds for the utility of the mechanism. Finally, we discuss the relationship between posterior sampling, the exponential mechanism and the \emph{safe Bayesian} generalisation of Bayesian inference. Follow-up work includes: \citet{smola:icml} who, under similar assumptions proved differential privacy results for Gibbs samplers; \citet{posterior-thesis} who improved some of our original bounds and also presented new results for other members of the exponential family; and \citet{AAAI16} who recently initiated the exploration of the posterior sampler in probabilistic graphical models on multiple random variables.

Another important follow up work is that of \citet{dwork2015reusable}.
 They have shown that \emph{any} differentially private algorithm results in robustness, in the sense that the divergence between posterior distribution arising from similar data is small. This has a direct impact on the generalisation ability of statistical models and inferences drawn, and consequently allows for what they call the ``re-usable hold-out''.
In our work, on the other hand, we have shown that with the right choice of prior, Bayesian inference is both private and robust. We have also shown that if the posterior distribution is robust, then it is also differentially private. In conclusion, robustness and privacy appear to be deeply linked, as our works have jointly shown conditions when one implies the other in three different ways: not only the same sufficient conditions can achieve both privacy and robustness, but privacy can also imply robustness, and robustness implies privacy. Further links between the two concepts are likely, as explained in the next section.

\subsection{Future Directions}
\label{sec:future-directions}
Although we have shown how Bayesian inference can already be differentially private by appropriately setting the prior,
we have not examined how this affects learning. While larger $c$ improves privacy, it also concentrates the prior so much that learning would be inhibited. Thus, $c$ could be chosen to optimise the trade-off between privacy and learning. However, we believe that the choice of the number of samples is easier to control.

\iffalse
From the theoretical side, we believe that the constant $\constFamily$
could be substantially improved, since right now it seems to be rather
loose. It is also possible that its existence is only an artefact of
the analysis, since it only appears for
Assumption~\ref{ass:hoelder-measure-observations}. However, we thought
it crucial to include these results in the paper, since they are
connected to the second necessary condition.
\fi

Other future directions include investigating the links between posterior sampling and the exponential mechanism, as well as with the safe Bayesian approach~\citep{grunwald2012safe} to inference. Consider an exponential mechanism which, given a utility function $u : \Params \times \Queries \to \Reals$ and a base measure $\mu$ on $\Params$ returns $\param \in \Params$ sampled from the density
\[
f(\theta) \propto e^{\epsilon u(\theta, \query)} \frac{\dd \mu(\param)}{\dd \lambda}\enspace.
\]
As also noted by~\citet{smola:icml}, this has a similar form to the posterior distribution, by setting $u(\theta, \query) = \ln p_\theta(x)$ and setting $\mu = \bel$ to the prior. This idea was used independently by~\citet{AAAI16} for releasing MAP point estimates. In this framework, privacy is achieved by setting $\epsilon$ to a sufficiently small value. However, it is interesting to note that this is how~\citet{grunwald2012safe} obtains robustness results for modified Bayesian inference. This implies that in some cases we can gain both privacy and efficiency. We note that in our case, we have proven that privacy is attainable by altering the prior, which corresponds to the base measure in the exponential mechanism. Consequently, we believe it is worthwhile examining settings where adjusting both $\epsilon$ and the prior measure may be advantageous.

\paragraph{Acknowledgments.} We gratefully thank Aaron Roth, Kamalika Chaudhuri, and Matthias Bussas for their discussion and insights as well as the anonymous reviewers for their comments on the paper, which helped to improve it significantly. This work was partially supported by the Marie Curie Project ``Efficient Sequential Decision Making Under Uncertainty'', Grant Number 237816; the People Programme (Marie Curie Actions) of the European Union's Seventh Framework Programme (FP7/2007-2013) under REA grant agreement n° 608743; the SNSF Project, ``SwissSenseSynergia''; and the Australian Research Council (DE160100584).

%%% Local Variables:
%%% mode: latex
%%% TeX-master: "jmlr"
%%% End:

\appendix
\section{Proofs of Main Results}
\label{sec:proofs}

\begin{proofof}{Lemma~\ref{lma:prodDist}}
  For Assumption~\ref{ass:hoelder-observations}, the proof follows directly from the definition of the absolute log-ratio distance; namely,
  \begin{align*}
    |\ln p_{\param}^n(\set{x_i}) - \ln p_{\param}^n(\set{y_i})| &\leq \tsum_{i=1}^{n}|\ln p_{\param}(x_i) - \ln p_{\param}(y_i)| \\
    & \le L \tsum_{i=1}^{n}{\xdist{x_i}{y_i}} \enspace.
  \end{align*}
  % This can be reduced from $n$ to $k$ if only $k$ items differ since
  % $\lrdist{p_{\param}(x_i)}{p_{\param}(y_i)}=0$ if $x_i=y_i$.

  For Assumption~\ref{ass:hoelder-measure-observations}, consider sub-family $\Params_L$ from Eq.~\eqref{eq:hoelder-measure-set} for marginal $p_{\param}$ and pseudo-metric $\xdistChar$, and define the corresponding sub-family $\Params^n_L$ %(or $\Params_{L \cdot k}$ for the $k$ differing items case)
  in terms of product distribution $p_{\param}^n$ and pseudo-metric $\xdistChar^n$.
  Then the same argument as above shows that $\Params_L\subseteq\Params_L^n$. Hence, the same prior and parameter $c$ yield the lower bound of Eq.~\eqref{eq:hoelder-measure-observations}, for $\Params_L^n$. %(or $\frac{c}{k}$)
\end{proofof}

\begin{proofof}{Theorem~\ref{the:kl}}
  Let us now tackle claim~\ref{the:kl-1}.
  First, we can decompose the KL-divergence into two parts.
  \begin{eqnarray}
    \dist{\bel(\cdot \mid x)}{\bel(\cdot \mid y)} % \nonumber \\
    &=&
    \int_{\Params}
    \ln \frac{\dd{\bel}(\param \mid x)}{\dd{\bel}(\param \mid y)} \dd{\bel}(\param \mid x) \nonumber \\
    &=&
    \int_{\Params}
    \ln
    \frac{p_\param(x)}{p_\param(y)}
    \dd{\bel}(\param \mid x)
    +
    \int_{\Params}
    \ln
    \frac{\marg(y)}{\marg(x)}
    \dd{\bel}(\param \mid x)
    \nonumber \\
    &\leq&
    \int_{\Params}
    \left|
      \ln
      \frac{p_\param(x)}{p_\param(y)}
    \right|
    \dd{\bel}(\param \mid x)
    +
    \int_\Params
    \ln
    \frac{\marg(y)}{\marg(x)}
    \dd{\bel}(\param \mid x)
    \nonumber \\
    &\leq&
    L\xdist{x}{y}
    +
    \left|\ln \frac{\marg(y)}{\marg(x)}\right|\enspace.
    \label{eq:kl-posterior-a}
  \end{eqnarray}

  From Assumption~\ref{ass:hoelder-observations},
  $p_\param(y) \leq \exp(L\xdist{x}{y}) p_\param(x)$ for all $\param$ so:
  \begin{align*}
    \marg(y)
    &= \int_\Params p_\param(y) \dd{\bel}(\param)
    \nonumber \\
    & \leq
    \exp(L\xdist{x}{y}) \int_\Params p_\param(x) \dd{\bel}(\param)
    % \\ &
    =
    \exp(L\xdist{x}{y}) \marg(x)\enspace.
  \end{align*}
  Combining this with \eqref{eq:kl-posterior-a} we obtain
  \begin{equation*}
    %\label{eq:kl-posterior-bound}
    \dist{\bel(\cdot \mid x)}{\bel(\cdot \mid y)}
    \leq
    2L\xdist{x}{y}\enspace.
  \end{equation*}
  Claim~\ref{the:kl-2} is dealt with similarly.
  Once more, we can break down the distance in parts. In more detail, we first write:
  \begin{align*}
    \dist{\bel(\cdot \mid x)}{\bel(\cdot \mid y)}
    &\leq
    \underbrace{
      \int_{\Params}
      \left|
        \ln
        \frac{p_\param(x)}{p_\param(y)}
      \right|
      \dd{\bel}(\param \mid x)
    }_{A}
    +
    \underbrace{
      \int_{\Params}
      \ln
      \frac{\marg(y)}{\marg(x)}
      \dd{\bel}(\param \mid x)
    }_{B}\enspace,
  \end{align*}
  as before. Now, let us re-write the $A$ term as
  \begin{align*}
    \int_{\Params}
    \left|
      \ln
      \frac{p_\param(x)}{p_\param(y)}
    \right|
    \frac{p_\param(x)}{\marg(x)}
    \dd{\bel}(\param)
    &\leq
    \sup_{\param'}\frac{p_{\param'}(x)}{\marg(x)}
    \int_{\Params}
    \left|
      \ln
      \frac{p_\param(x)}{p_\param(y)}
    \right|
    \dd{\bel}(\param)\enspace,
  \end{align*}
  so that the left-hand side term is the ratio between the maximal likelihood and marginal likelihood.  Using the same steps, we can bound $B$ in the same manner.

Now, let us define a data-dependent and a data-independent bound:
  \begin{align*}
    \constFamily(x)
    &\defn
    \sup_{\param} \frac{p_\param(x)}{\marg(x)}\enspace,
    &
    \constFamily &\defn \sup_x \constFamily(x)\enspace.
%    \label{eq:magic-constant}
  \end{align*}
  Replacing, we obtain:
  \begin{align*}
    \dist{\bel(\cdot \mid x)}{\bel(\cdot \mid y)}
    &\leq
    \constFamily
    \underbrace{
      \int_{\Params}
      \left|
        \ln
        \frac{p_\param(x)}{p_\param(y)}
      \right|
      \dd{\bel}(\param)
    }_{A}
    +
    \underbrace{
      \int_{\Params}
      \ln
      \frac{\marg(y)}{\marg(x)}
      \dd{\bel}(\param \mid x)
    }_{B}\enspace.
  \end{align*}

  Now, to bound the individual terms, we start from $A$ and note that theorem 3 of~\citep{norkin1986stochastic} on the Lipschitz property of the expectation of stochastic Lipschitz functions applies.
  \begin{theorem}\citep{norkin1986stochastic}
    If  $\bel$ is a probability measure on $\Params$ and $f : \CS \times \Params \to \Reals$ is a $\bel$-measurable function, such that for any $\param \in \Params$, $f(\cdot, \param)$ is $\ell(\param)$-Lipschitz, then the function $f_\bel(x) \defn \E_\bel f(x, \param)$ is $L_\bel$-Lipschitz, where $\L_\bel = \E_\bel \ell(\param)$.
  \end{theorem}

Recall that the expectation of a non-negative random variable can be written in terms of its CDF $F$ as $\int_0^\infty [1 - F(t)] \dd t$.
In our case, $\ell(\param)$ is a random variable on $\Params$, and we can write its cumulative distribution function as
\[
F(t) \defn \bel\left(\cset{\param \in \Params}{\ell(\param) \leq t}\right) = \bel(\Params_t)\enspace,
\]
by the definition of $\Params_t$.
It follows that $\ln p_\param(x)$ is $L_\bel$-Lipschitz, where through the formula for the expectation of positive variables:
\begin{align}
%L_\bel
%&= \int_0^\infty [1 - \bel(\Theta_t)] \dd t
%\leq \int_0^\infty e^{-ct} \dd t
%= c^{-1}.
%\label{eq:expected-L}
%\\
L_\bel
&= \int_0^\infty [1 - \bel(\Theta_t)] \dd t
\leq
L_0 \bel(\Params_{L_0})+
  [1 - \bel(\Params_{L_0})] \int_{0}^\infty e^{-ct} \dd t
\leq
L_0  + c^{-1}\enspace.
\label{eq:expected-L}
\end{align}
So, term $A$ becomes $\constFamily \left(L_0 + c^{-1} \right) \xdist{x}{y}$.

Now let us move on to term $B$. For technical reasons, we start by considering a pair $x, y$ such that $\xdist{x}{y} \leq c - 1$. This also implies  that $c > 1$, since the distance cannot be negative.
\begin{align}
  \frac{\marg(x)}{\marg(y)}
  &
    \overset{(a)}{=} \int_\Params \frac{p_\param(x)}{\marg(y)} \dd \bel(\param)
    \overset{(b)}{\leq} \int_\Params \frac{p_\param(y) e^{\ell(\param) \xdist{x}{y}}}{\marg(y)} \dd \bel(\param)
    \overset{(c)}{\leq} \constFamily \int_\Params e^{\ell(\param) \xdist{x}{y}} \dd \bel(\param) \enspace.
\end{align}
Note that $\cset{\param \in \Params}{e^{\ell(\param) \xdist{x}{y}} \leq t} = \cset{\param \in \Params}{\ell(\param) \leq \xdist{x}{y}^{-1} \ln t} = \Params_{\xdist{x}{y}^{-1} \ln t}$. So the CDF of the random variable $e^{\ell(\param)}$ is $F(t) = \bel(\Params_{\xdist{x}{y}^{-1} \ln t})$. Then:

\iffalse
\comment{Just modifying the previous version to see what we get:}
\begin{align}
\E_\bel e^{\ell(\theta) \xdist{x}{y}}
&=
\int_0^\infty [1 - \bel(\Params_{\xdist{x}{y}^{-1} \ln t})]
\dd t
\\
&\leq
e^{L_0} + 
\int_{t_0}^\infty [1 - \bel(\Params_{\xdist{x}{y}^{-1} \ln t})]
\dd t,
& t_0 & = e^{\xdist{x}{y} L_0}
\\
&\leq
e^{L_0} + \int_{t_0}^\infty e^{- \xdist{x}{y}^{-1} c \ln t} \dd t \
 =
e^{L_0} + \int_{t_0}^\infty t^{-\xdist{x}{y}^{-1} c} \dd t\\
&=
e^{L_0} + \left[
\frac{t^{1-c \xdist{x}{y}^{-1}}}{1-c \xdist{x}{y}^{-1}}
\right]_{t_0}^\infty
=
e^{L_0} + \frac{t_0^{1-c \xdist{x}{y}^{-1}}}{c \xdist{x}{y}^{-1} - 1}
\\
&=
e^{L_0} + \frac{e^{L_0(\xdist{x}{y} - c)}}{c \xdist{x}{y}^{-1} - 1}
\leq
e^{L_0} + e^{L_0(\xdist{x}{y} - c)} \xdist{x}{y}
\\
&\leq
e^{L_0} + e^{L_0} \xdist{x}{y}
=
e^{L_0}(1 +  \xdist{x}{y})
\leq e^{L_0}e^{\xdist{x}{y}}\enspace.
\end{align}
Consequently, $\ln \marg(x)/\marg(y) \leq \constFamily (L_0 + \xdist{x}{y})$. 
\comment{The $L_0$ term there seems inevitable the way I do this.}
\fi

%\comment{An alternative method relies on the definition of the expectation of powers, below:}
For positive random variables, $\E X^\xdistChar = \xdistChar \int_0^\infty t^{\xdistChar - 1} [1 - F(t)] dt$. Applying this to our case, we get:
\begin{align*}
\E_\bel e^{\ell(\theta) \xdist{x}{y}}
&=
\E_\bel [e^{\ell(\theta) \xdist{x}{y}} \mid \ell \leq L_0] \bel(\Params_{L_0})
+
\E_\bel [e^{\ell(\theta) \xdist{x}{y}} \mid \ell > L_0] [1 - \bel(\Params_{L_0})]
\\
&\leq
e^{L_0 \xdist{x}{y}}
+
\xdist{x}{y} \int_{t_0}^\infty t^{\xdist{x}{y} - 1} [1 - \bel(\Params_{\ln t})] \dd t\\
    \tag{where $t_0 = e^{L_0}$}
&\leq
e^{L_0 \xdist{x}{y}}
+
\xdist{x}{y} \int_{t_0}^\infty e^{\ln t[\xdist{x}{y} - 1]} e^{-c(\ln t - L_0)} \dd t
\\
&=
e^{L_0 \xdist{x}{y}}
+
\xdist{x}{y} \int_{t_0}^\infty e^{\ln t[\xdist{x}{y} - c - 1] + cL_0} \dd t
\\
&=
e^{L_0 \xdist{x}{y}}
+
\xdist{x}{y} e^{c L_0} \int_{t_0}^\infty t^{\xdist{x}{y} - c - 1} \dd t
\\
&=
e^{L_0 \xdist{x}{y}}
+
\xdist{x}{y} e^{c L_0}
\frac{t_0^{\xdist{x}{y} - c}}{c - \xdist{x}{y} }
\\
&=
e^{L_0 \xdist{x}{y}}
+
\xdist{x}{y} e^{c L_0}
\frac{e^{L_0(\xdist{x}{y} - c)}}{c - \xdist{x}{y} }
\\
&\leq
e^{L_0 \xdist{x}{y}}
+
\xdist{x}{y} e^{c L_0}
e^{L_0(\xdist{x}{y} - c)}
\\
&=
e^{L_0 \xdist{x}{y}}
+
\xdist{x}{y} e^{L_0 \xdist{x}{y}}
=
(1 + \xdist{x}{y}) e^{L_0 \xdist{x}{y}}
\leq
e^{(1 + L_0) \xdist{x}{y}}.
\end{align*}
Consequently, $\ln \marg(x)/\marg(y) \leq \constFamily (1 + L_0) \xdist{x}{y}$. %\comment{This seems to be the best way. For $L_0 = 0$ we obtain the original bound.}

\iffalse
\comment{Previous version:
\begin{align}
\E_\bel e^{\ell(\theta) \xdist{x}{y}}
&\leq
\int_0^\infty [1 - \bel(\Params_{\xdist{x}{y}^{-1} \ln t})]
\dd t
\leq
1 +
\int_1^\infty [1 - \bel(\Params_{\xdist{x}{y}^{-1} \ln t})]
\dd t\\
&\leq
1 + \int_1^\infty e^{- \xdist{x}{y}^{-1} c \ln t} \dd t \
 =
1 + \int_1^\infty t^{-\xdist{x}{y}^{-1} c} \dd t\\
&=
1 + \left[
\frac{t^{1-c \xdist{x}{y}^{-1}}}{1-c \xdist{x}{y}^{-1}}
\right]_1^\infty
=
\frac{c}{c-\xdist{x}{y}}
\leq
1 + \rho(x, y),
\end{align}
where the integral converges since $c > \xdist{x}{y}$ and the last
inequality follows from the fact that $\rho < c - 1$.  Then note
that $\ln \rho + 1 \leq \rho$. Consequently $\ln \marg(x)/\marg(y) \leq \constFamily \xdist{x}{y}$. 
}
\fi
To handle
larger distances $\rho$, we can simply apply the above result
repeatedly between $k$ datasets $z_1, \ldots, z_{k}$, where $z_1 = x$,
$z_k = y$ and such that $\xdist{z_i}{z_{i+1}} < c - 1$.\footnote{Technically, the dataset space is a complete metric space for the intermediate points to exist.}  By chaining
logarithmic ratios, \ie using the fact that
$\ln \marg(x) / \marg(y) = \ln \marg(x) / \marg(z) + \ln \marg(z) /
\marg(y)$ we can now extend our result to general pairs for term $B$.
Replacing those terms, we now obtain the final result.
\[
\dist{\bel(\cdot \mid x)}{\bel(\cdot \mid y)}
\leq
\constFamily 
\left(
  1 + 2L_0 + c^{-1}
\right)
\xdist{x}{y}\enspace.
\]
\iffalse
\comment{Results without $L_0$: Note the new one should be an upper bound if $L_0 = 0$.
Replacing those terms, we now obtain the final result.
\[
\dist{\bel(\cdot \mid x)}{\bel(\cdot \mid y)}
\leq
\left\{
\constFamily
c^{-1}
+
\ln
\constFamily
\right\}
\xdist{x}{y}\enspace.
\]
}
\fi
\end{proofof}

\begin{proofof}{Theorem~\ref{thm:dp}}
  For part~\ref{thm:dp1}, we assumed that there is an $L>0$ such that $\forall x, y \in \mathcal{\CS}$,  $\left| \log \frac{p_\param(x)}{p_\param(y)} \right| \leq L \xdist{x}{y}$, thus implying $\frac{p_\param(x)}{p_\param(y)} \leq \exp\{ L \xdist{x}{y} \}$. Further, in the proof of Theorem~\ref{the:kl}, we showed that $\marg(y) \le \exp\{ L \xdist{x}{y} \} \marg(x)$ for all $x,y \in \mathcal{\CS}$. From Eq.~\eqref{eq:posterior}, we can then combine these to bound the posterior of any $B \in \field{\Params}$ as follows for all $x,y \in \CS$:
  \begin{align*}
    \bel(B \mid x)
    % & = \frac{\int_B p_\param(x) \dd{\bel}(\param)}{\marg(x)} \\
    & = \frac{\int_B \frac{p_\param(x)}{p_\param(y)} p_\param(y) \dd{\bel}(\param)}{\marg(y)} \cdot \frac{\marg(y)}{\marg(x)}
     \le \exp\{2L \xdist{x}{y} \} \bel(B \mid y) \enspace.
  \end{align*}

  For part~\ref{thm:dp2}, note that from Theorem~\ref{the:kl-2} that the KL divergence of the posteriors under assumption is bounded by \eqref{eq:kl-2}. Now, recall Pinsker's inequality~\citep[cf.][]{fedotov2003refinements}:
  \begin{equation}
    \label{eq:pinsker}
    \KL{Q}{P} \geq \frac{1}{2}\onenorm{Q-P}^2\enspace.
  \end{equation}
  This  yields:
  $\abs{\bel(B \mid x) - \bel(B \mid y)}
  \leq \sqrt{\frac{1}{2}\dist{\bel(\cdot \mid x)}{\bel(\cdot \mid y)}}
  \leq \sqrt{\frac{1}{2} \constFamily \left(1 + 2L_0 c^{-1}\right) \xdist{x}{y}}$.
\end{proofof}

\begin{proofof}{Lemma~\ref{lem:pac-utility}}
  Sampling $\nsamples$ times from the posterior, gives us the following estimate of the utility function $\hat{\util_\bel}(\query, \response) = \frac{1}{\nsamples} \sum_{\param \in \hat{\Params}} \util_\param(\query, \response),$
which with probability at least $1 - \delta$ satisfies $| \hat{\util_\bel}(q,r) - u(q,r) |  < \sqrt{\frac{\ln(2/\delta)}{2\nsamples}} = \epsilon$, $\forall r, q$, via Hoeffding's inequality and the boundedness of $\util$.
Consequently, we can be at most $2 \epsilon$-away from the optimal.
\end{proofof}

\begin{proofof}{Lemma~\ref{lem:l1-empirical-error}}
  (Note that in this proof, $\varepsilon, \delta$ do not refer to the privacy parameters.)
  We use the inequality due to~\citet{weissman2003inequalities} on the $\ell_1$ norm, which states that for any multinomial distribution $P$ with $m$ outcomes, the $\ell_1$ deviation of the empirical distribution $\hat{P}_n$ after $n$ draws from the multinomial satisfies:
  \begin{equation*}
    %\label{eq:weissman}
    \Pr\left(\onenorm{\hat{P}_n - P} \geq \varepsilon\right) \leq (2^m -2) e^{-\frac{1}{2} n \varepsilon^2},
    \qquad
    \forall \varepsilon > 0\enspace.
  \end{equation*}
  The right hand side is bounded by $e^{m \ln 2 - \frac{1}{2}n\varepsilon^2}$.
  Substituting $\varepsilon =  \sqrt{\frac{3}{n} \ln \frac{1}{\delta}}$:
  \begin{align*}
	  \Pr\left(\onenorm{\hat{P}_n - P}  \geq  \sqrt{\frac{3}{n} \ln \frac{1}{\delta}}\right)
	  &\leq
    e^{m \ln 2 - \frac{3}{2} \ln \frac{1}{\delta}}
    \\ \nonumber
    &\leq
    e^{\log_2 \sqrt{\frac{1}{\delta}} \ln 2 - \frac{3}{2} \ln \frac{1}{\delta}} \\ \nonumber
    &=
    e^{\frac{1}{2} \ln \frac{1}{\delta} - \frac{3}{2} \ln \frac{1}{\delta}} \nonumber \\
    &= \delta\enspace.
  \end{align*}
  where the second inequality follows from $m  \leq \log_2 \sqrt{1/\delta}$.
\end{proofof}

\begin{proofof}{Theorem~\ref{the:adversary-sample-complexity}}
  Recall that the data processing inequality states that, for any sub-algebra $\algebra$:
  \begin{align*}
    %\label{eq:data-processing}
    % \KL{Q_{|\algebra}}{P_{|\algebra}} & \leq \KL{Q}{P}\\
    \onenorm{Q_{|\algebra} - P_{|\algebra}} & \leq \onenorm{Q - P}\enspace.
  \end{align*}
  Using this and Pinsker's inequality \eqref{eq:pinsker} we obtain:
  \begin{align*}
    2L\xdist{x}{y}
    & \geq
%    2L\epsilon
%    \geq
    \KL{\bel(\cdot \mid x)}{\bel(\cdot \mid y)} \nonumber \\
    & \geq
    \frac{1}{2}\onenorm{\bel(\cdot \mid x) - \bel(\cdot \mid y)}^2 \nonumber
    \\
    &\geq
    \frac{1}{2}\onenorm{\bel_{|\algebra}(\cdot \mid x) - \bel_{|\algebra}(\cdot \mid y)}^2\enspace.
  \end{align*}
  On the other hand, due to \eqref{eq:l1-empirical-error} the adversary's $\ell_1$ error in the posterior distribution is bounded by $\sqrt{\frac{3}{n} \ln \frac{1}{\delta}}$ with probability $1 - \delta$. In order for him to be able to distinguish the two different posteriors, it must hold that
  \begin{equation*}
    \label{eq:necessary-distinguishability}
    \onenorm{\bel_{|\algebra}(\cdot \mid x) - \bel_{|\algebra}(\cdot \mid y)} \geq \sqrt{\frac{3}{n} \ln \frac{1}{\delta}}\; .
  \end{equation*}
 Using the above inequalities, we can bound the error in terms of the distinguishability of the real dataset $x$ from an arbitrary set $y$ as:
  \begin{equation*}
    4 L \xdist{x}{y} \geq \frac{3}{n} \ln \frac{1}{\delta}\enspace.
  \end{equation*}
  Rearranging, we obtain the required result. The second case is treated similarly to obtain:
  \begin{equation*}
    \left(\constFamily c^{-1} + \ln \constFamily\right) \xdist{x}{y} / 2 \geq \frac{3}{n} \ln \frac{1}{\delta}\enspace.
  \end{equation*}

\end{proofof}

\begin{proofof}{Lemma~\ref{lem:utility-bound}}
  Let $\response, \response^\star$ be the optimal responses under $\bel, \bel^\star$ respectively. For notational convenience, let $\util_\bel = \int_\Params \util_\param \dd \bel(\param)$ denote the expected utility under a belief $\bel$. Then our regret is
  \begin{align*}
    \util_{\bel}(\query, \response)
    -
    \util_{\bel}(\query, \response^\star)
    &=
    \util_{\bel}(\query, \response)
    -
    \util_{\bel^\star}(\query, \response)
    \\
    &+
    \util_{\bel^\star}(\query, \response)
    -
    \util_{\bel^\star}(\query, \response^\star)
    \\
    &+ \util_{\bel^\star}(\query, \response^\star)
    - \util_{\bel}(\query, \response^\star)
    \\
    &\leq
    2 \norm{\bel - \bel^\star}_1\enspace.
  \end{align*}
  This follows from the fact that
  \begin{align*}
    \util_{\bel}(\query, \response) - \util_{\bel^\star}(\query, \response)
    &=
    \int_\Params \util_{\param}(\query, \response) \dd [\bel - \bel^\star](\param)\\
    &\leq \|\util\|_\infty \|\bel -\bel^\star\|_1
  \end{align*}
  and then using the boundedness of $\util$. The third term is dealt with identically.
  For the second term, note that $\util_{\bel^\star}(\query, \response) - \util_{\bel^\star}(\query, \response^\star) \leq 0$ since $\response^\star$ maximises $\util_{\bel^\star}$.
\end{proofof}

\begin{proofof}{Lemma~\ref{lem:expected-kl-divergence}}
  Let $\marg^\star(x) = \int_\Params p_\param(x) \dd{\bel^\star}(x)$ be the prior marginal distribution. Then the $\bel^\star$-expected KL divergence between the two posteriors is
  \begin{align*}
    &\sum_x \int_\Params
    \ln \frac{\dd \bel^\star(\param \mid x)}{\dd \bel(\param \mid x)}
    \dd \bel^\star(\param \mid x) \marg^\star(x)
    \\
    &\leq
    \sum_x \int_\Params
    \left(
    \left|\ln \frac{\dd \bel^\star(\param)}{\dd \bel(\param)}\right|
    +
    \left|\ln \frac{\marg(x)}{\marg^\star(x)}\right|
    \right)
    \dd \bel^\star(\param \mid x) \marg^\star(x)
    \\
    &\leq 2\eta\enspace.
  \end{align*}
  The first term $\left|\ln \frac{\dd \bel^\star(\param)}{\dd \bel(\param)}\right|$ is bounded by $\eta$ by assumption. From the same assumption, it follows that
  $\marg(x) = \int_\Params p_\param(x) \dd{\bel}(\param)
  \leq \int_\Params p_\param(x) e^\eta \dd{\bel^\star}(\param) = e^\eta \marg^\star(x)$, and so the second term is also bounded by $\eta$.
\end{proofof}

%%% Local Variables:
%%% mode: latex
%%% TeX-master: "jmlr.tex"
%%% End:

\section{Proofs of Examples}
\label{sec:example-proofs}
\begin{proofof}{Lemma~\ref{lem:exponential}}
Since $Exp(x;\theta)$ is monotonic decreasing in $x$ and concave as a function of $\theta$, we have $\inf_{\{||x||\leq B,\theta\in[c_{1},c_{2}]\}} Exp(x;\theta)=\min\left\{c_{1}e^{-c_{1}B},c_{2}e^{-c_{2}B}\right\}\leq\phi(x)$. Then we have
$$\constFamily=c_{2}/\min\left\{c_{1}e^{-c_{1}B},c_{2}e^{-c_{2}B}\right\}\enspace.$$
Next we compute the absolute log-ratio
distance for any $x_1$ and $x_2$ according to the exponential
likelihood function:
\[
|\ln p_{\param}(x_1) - \ln p_{\param}(x_2)| = \param | x_1 - x_2 | \enspace.
\]
Thus, for $\theta\in [c_{1},c_{2}]$, under Assumption~\ref{ass:hoelder-measure-observations}, using $\xdist{x}{y}=|x-y|$, the set
of feasible parameters for any $L> c_{1}$ is $\Params_L = (c_{1},L)$. Note the density of the renormalized exponential prior on $[c_{1},c_{2}]$ is given by
$K\lambda e^{-\lambda \theta}$, where $K=(e^{-\lambda c_{1}}-e^{-\lambda c_{2}})^{-1}$. Thus the CDF at $L$ of this density is $K\left(e^{-\lambda c_{1}}-e^{-\lambda L}\right)$ for $L\in[c_{1},c_{2}]$ and $1$ for $L\geq c_{2}$. It is natural to choose $L_{0}$ to be $c_{1}$. Then we need to find $c$ such that
$$\xi(\Theta_{L})=\int^{L}_{c_{1}}K\lambda e^{-\lambda \theta}d\theta=K(e^{-\lambda c_{1}}-e^{-\lambda L})\geq 1- e^{-c(L-c_{1})}$$
for $L\in (c_{1},c_{2})$. By plugging $K$ into the inequality, we have
\begin{align*}
 e^{-c(L-c_{1})}\geq \frac{e^{-\lambda(L-c_{2})}-1}{e^{-\lambda(c_{1}-c_{2})}-1}\enspace.
\end{align*}
Since $e^{-\lambda(L-c_{2})}\leq e^{-\lambda(c_{1}-c_{2})}$, it is sufficiency to find $c$ such that $e^{-c(L-c_{1})}\geq e^{-\lambda(L-c_{1})}$. Therefore we can have $c=\lambda$.
%Therefore the assumption requires the prior to adequately support this range, but
%because the CDF at $L$ of the trimmed exponential prior with parameter $\lambda >0$
%is given by $1 -\exp(-\lambda L)$ for $L\in[c_{1},c_{2})$ and $1$ for $L\geq c_{2}$, every such prior satisfies the
%assumption with $c = \lambda$, $L_{0}=c_{1}$.
\end{proofof}

\begin{proofof}{Lemma~\ref{lem:laplace}}
Note that $Laplace(x;s,\mu)$ is monotonic decreasing in $x$ if $x< \mu$, and increasing in $x$ if $x\geq \mu$. Since $Laplace(x;s,\mu)$ is concave as a function of $s$,  we have $\phi(t)\geq\min\left\{\frac{1}{2c_{2}},\frac{1}{2c_{1}}\exp\left(\frac{-B-\mu}{c_{1}}\right)\right\}$ if $x<\mu$ and $\phi(t)\geq\min\left\{\frac{1}{2c_{2}},\frac{1}{2c_{1}}\exp\left(\frac{\mu-B}{c_{1}}\right)\right\}$ if $x\geq \mu$. Thus, we can take
$$\constFamily =\begin{cases} \frac{c_{2}}{2\min\left\{\frac{1}{2c_{2}},\frac{1}{2c_{1}}\exp\left(\frac{-B-\mu}{c_{1}}\right)\right\}}\enspace, & x < \mu \\
                    \frac {c_{2}}{2\min\left\{\frac{1}{2c_{2}},\frac{1}{2c_{1}}\exp\left(\frac{\mu-B}{c_{1}}\right)\right\}}\enspace, &   x\geq \mu
       \end{cases}\enspace.$$
For any $x_1$ and $x_2$, the
absolute log-ratio distance for this distribution can be bounded as
\begin{align*}
& |\ln p_{\mu,s}(x_1) - \ln p_{\mu,s}(x_2)| \\
= & \tfrac{1}{s}\left| \|x_1-\mu\| - \|x_2-\mu\| \right| \le \tfrac{1}{s}\|x_1-x_2\| \enspace,
\end{align*}
where the inequality follows from the triangle inequality on
$\| \cdot \|$. Thus, if we use $\xdist{x}{y}=\|x-y\|$, the set of
feasible parameters for
Assumption~\ref{ass:hoelder-measure-observations} is $\mu \in \Reals$
and $\frac{1}{s} =\theta \le L$. Again we can use the trimmed exponential prior with rate
parameter $\lambda>0$ for the inverse scale, $\frac{1}{s}$, and %any prior on
%$\mu$ to obtain the second part of Assumption~\ref{ass:hoelder-measure-observations}.
 similar to the previous example, Assumption~\ref{ass:hoelder-measure-observations} is satisfied with $c=\lambda$ and $L_{0}=c_{1}$.
 %The ratio between the maximum and marginal likelihoods here is same as the ratio in exponential conjugate prior example. These similarities
%are not surprising considering that if $X \sim Laplace(\mu,s)$ then
%$\|X-\mu\| \sim Exponential(\frac{1}{b})$.
\end{proofof}

\begin{proofof}{Lemma~\ref{lem:beta}}
Here, we consider data drawn from a Binomial distribution with a
beta prior on its proportion parameter, $\param$. Thus, the
likelihood and prior functions are
\begin{align*}
 p_{\param,n}(X=k) & = \tbinom{n}{k}
 \param^{k}(1-\param)^{n-k} \\
 \bel_0(\param) &= \tfrac{1}{B(a,b)}\param^{a-1}(1-\param)^{b-1}
 \enspace,
\end{align*}
where $k \in \{0,1,2,\ldots,n\}$, $a,b \in \Reals_+$ and $B(a,b)$ is
the beta function. The resulting posterior is a Beta-Binomial
distribution. Again we consider the application of
Assumption~\ref{ass:hoelder-measure-observations} to this
Beta-Binomial distribution. For this purpose, we must quantify the
parameter sets $\Params_L$ for a given $L > 0$ according to a distance
function. The absolute log-ratio distance between the Binomial
likelihood function for any pair of arguments, $k_1$ and $k_2$, is
\begin{align*}
 |\ln p_{\param,n}(k_1) - \ln p_{\param,n}(k_2)| =
 \left|\Delta_{n}(k_1,k_2) +
 (k_1 - k_2) \ln \tfrac{\param}{1-\param} \right|
\end{align*}
where $\Delta_{n}(k_1,k_2) \defn \ln \binom{n}{k_1} - \ln
\binom{n}{k_2}$. By substituting this distance into the supremum of
Eq.~\eqref{eq:hoelder-measure-set}, we seek feasible values of $L > 0$
for which the supremum is non-negative; here, we explore the case where
$\xdist{(n, k_1)}{(n, k_2)} \defn |k_1 - k_2|$.
Without loss of generality, we assume $k_1 > k_2$, and
thus require that
\begin{equation}
  \sup_{k_1 > k_2}{\left|\tfrac{\Delta_{n}(k_1,k_2)}{k_1 - k_2} + \ln \tfrac{\param}{1-\param} \right|} \le
  L
  \enspace.
  \label{eq:bino-bound}
\end{equation}
However, by the definition of $\Delta_{n}(k_1,k_2)$, the ratio
$\frac{\Delta_{n}(k_1,k_2)}{k_1 - k_2}$ is in fact the slope of the
chord from $k_2$ to $k_1$ on the function $\ln \binom{n}{k}$. Since
the function $\ln \binom{n}{k}$ is concave in $k$, this slope
achieves its maximum and minimum at its boundary values; \ie it is
maximised for $k_1=1$ and $k_2=0$ and minimised for $k_1=n$ and
$k_2=n-1$. Thus, the ratio attains a maximum value of $\ln n$ and a
minimum of $- \ln n$ for which the above supremum is simply $\ln n +
\left|\ln \frac{\param}{1-\param} \right|$. From Eq.~\eqref{eq:bino-bound}, we therefore have, for all $L \ge  \ln n$:
\begin{align*}
 \Params_L &= \left[ \left(1+\tfrac{e^{L}}{n}\right)^{-1},
   \left(1+\tfrac{n}{e^{L}}\right)^{-1}\right]
 \enspace.
\end{align*}

We want to bound $\bel(\Theta_{L})$. We know that:
$\bel(\Theta_{L})=1-\bel\left(\Theta^\complement_L\right)$ where $\Theta^\complement_L$ is the complement of $\Theta_L$.  so $\bel(\Theta^\complement_L)$ is composed of two symmetric intervals:
$\left[0, \left(1+\frac{e^L}{n}\right)^{-1}\right)$ and $\left(\left(1 + \frac{n}{e^L}\right)^{-1}, 1\right]$. We selected $\alpha = \beta$, therefore the mass must concentrate at $\frac{1}{2}$, as we have $\alpha > 1$.

Due to symmetry, the mass outside of $\Theta_L$ is two times that is the first interval. This is:
\[
\frac{2}{B(\alpha,\alpha)} \int_{0}^\frac{p}{1+p}   x^{\alpha-1} (1-x)^{\alpha-1} \dd x\; .
\]
where $p$ denotes $ne^{-L}\in [0,1]$,
Therefore $c$ is upper bounded by
$$\ln\left(\frac{2A(p)}{B(\alpha,\alpha)}\right)/(L_{0}-L)=\ln\left(\frac{2A(p)}{B(\alpha,\alpha)}\right)/\ln{p},$$
where $A(p)$ denotes the incomplete Beta function $\int_{0}^{\frac{p}{1+p}}x^{\alpha-1}(1-x)^{\alpha-1}dx$.
Note that we have
$$A'(p)=\frac{p^{\alpha-1}}{(1+p)^{2\alpha}}\; ,$$
$$A''(p)=\frac{p^{\alpha-2}[(\alpha-1)(1+p)-2\alpha p]}{(1+p)^{2\alpha+1}}\; .$$
\begin{claim}
$H(p)=\alpha A(p)-\frac{p^{\alpha}}{(1-p)(1+p)^{2\alpha-1}}\leq 0$ for all $p\in(0,1)$.
\end{claim}
\begin{proof}Calculating derivatives and simplifying
\begin{eqnarray*}
	&& H'(p)\\
	&=&\alpha A'(p)- \frac{\alpha p^{\alpha-1}(1-p)(1+p)^{2\alpha-1}-p^{\alpha}\left[(2\alpha-1)(1-p)(1+p)^{2\alpha-2}-(1+p)^{2\alpha-1}\right]}{[(1-p)(1+p)^{2\alpha-1}]^{2}}\nonumber\\
&=&\frac{\alpha p^{\alpha-1}}{(1+p)^{2\alpha}}-\frac{\alpha p^{\alpha-1}(1-p)(1+p)-p^{\alpha}[(2\alpha-1)(1-p)-(1+p)]}{(1-p)^{2}(1+p)^{2\alpha}}\nonumber\\
&=&\frac{p^{\alpha-1}}{(1+p)^{2\alpha}}\left(\alpha-\frac{\alpha(1-p^{2})-2p(\alpha-1-p\alpha)}{(1-p)^{2}}\right)\nonumber\\
&=&\frac{p^{\alpha-1}}{(1+p)^{2\alpha}(1-p)^{2}}\left(\alpha(1-2p+p^{2})-\alpha(1-p^{2})+2p(\alpha-1-\alpha p)\right)\nonumber\\
&=&\frac{-2p^{\alpha}}{(1+p)^{2\alpha}(1-p)^{2}}<0\; .\nonumber
\end{eqnarray*}
Therefore $H(p)$ is strictly decreasing. Then combined with $H(0)=0$, we claim follows.
\end{proof}
\begin{claim}
$G(p)=p\frac{A'(p)}{A(p)}\ln{p}-\ln{\frac{2A(p)}{B(\alpha,\alpha)}}<0$ for all $p\in(0,1)$.
\end{claim}
\begin{proof}Again taking derivatives
\begin{eqnarray*}
G'(p)&=&\frac{A'(p)}{A(p)}(1+\ln{p})+p\ln{p}\frac{A''(p)A(p)-A'(p)^{2}}{A(p)^2}-\frac{A'(p)}{A(p)}\\
&=&\frac{\ln{p}}{A(p)^{2}}(A(p)A'(p)+pA''(p)A(p)-pA'(p)^{2})\\
&=&\frac{\ln{p}}{A(p)^{2}}\left[\frac{p^{\alpha-1}}{(1+p)^{2\alpha}}A(p)\left(1+\frac{(\alpha-1)(1+p)-2\alpha p}{1+p}\right)-\frac{p^{2\alpha-1}}{(1+p)^{4\alpha}}\right]\\
&=&\frac{\ln{p}}{A(p)^{2}}\frac{p^{\alpha-1}}{(1+p)^{2\alpha+1}}\left[\alpha(1-p)A(p)-\frac{p^{\alpha}}{(1+p)^{2\alpha-1}}\right]\\
&=&\frac{p^{\alpha-1}}{(p+1)^{2\alpha+1}A(p)^{2}}H(p)\ln{p}(1-p)>0\; .
\end{eqnarray*}
So $G(p)$ is strictly increasing. Combined with $\lim_{p\rightarrow1}G(p)=0$, the claim follows.
\end{proof}
\begin{claim}
	$F(p)=\ln\left(2I_{\frac{p}{1+p}}(\alpha)\right)/\ln{p}$ is decreasing in $p\in (0,1)$, where the incomplete Beta function $I_{\frac{p}{1+p}}(\alpha)=A(p)/B(\alpha,\alpha)$.
\end{claim}
\begin{proof}Taking derivatives
\begin{eqnarray*}
F'(p)&=&\frac{1}{\ln^{2}{p}}\left(\frac{A'(p)}{A(p)}\ln{p}-\frac{1}{p}\ln{\frac{2A(p)}{B(\alpha,\alpha)}}\right)\\
&=&\frac{1}{p\ln^{2}{p}}\left(\frac{A'(p)}{A(p)}p\ln{p}-\ln{\frac{2A(p)}{B(\alpha,\alpha)}}\right)\\
&=&\frac{1}{p\ln^{2}{p}}G(p)<0\; .
\end{eqnarray*}
\end{proof}
Therefore $\ln\left(2I_{\frac{p}{1+p}}(\alpha)\right)/\ln{p}$ is monotonic decreasing in $p$.
Thus the minimum value of $F(p)$ is $\frac{1}{B(\alpha)2^{2\alpha-1}}$ as $p \rightarrow 1$, which we can take as our $c$ in this example.
%Then Assumption 2 is satisfied by taking $L_{0}$ as $\ln{n}$ and $c$ as $2^{-2\alpha+1}/B(\alpha)$.

Let us consider $\constFamily$ for this example.
We have
%$$\phi(x)=\binom{n}{k}\frac{B(\alpha+x, n-x+\beta){B(\alpha, \beta)},$$
%then we have
$$\frac{p_{\theta}(x)}{\phi(x)}=\frac{B(\alpha, \beta)\theta^{x}(1-\theta)^{n-x}}{B(\alpha+x,n+\beta-x)}\; ,$$
where $\theta\in [0,1]$ and $x\in [0,1, \ldots, n]$.
Note that
$$\frac{B(\alpha+x+1, n+\beta-x-1)}{B(\alpha+x, n+\beta-x)}=\frac{\Gamma(\alpha+x+1)\Gamma(n+\beta-x-1)}{\Gamma(\alpha+x)\Gamma(n+\beta+1)}=\frac{\alpha+x}{n+\beta-x-1}\enspace.$$
So $B(\alpha+x+1, n+\beta-x-1)\leq B(\alpha+x, n+\beta-x)$ if $x\leq \frac{n+\beta-\alpha-1}{2}$; $B(\alpha+x+1, n+\beta-x-1)> B(\alpha+x, n+\beta-x)$ otherwise. Thus
$$B(\alpha+x, n+\beta -x)\geq B\left(\frac{n+\alpha+\beta-1}{2},\frac{n+\alpha+\beta+1}{2}\right)\; .$$
Hence we can take $\constFamily=B(\alpha,\beta)/B\left(\frac{n+\alpha+\beta-1}{2},\frac{n+\alpha+\beta+1}{2}\right)$.
%Hence, the feasible set of parameters is a symmetric interval
%about $\frac{1}{2}$ for all $L \ge \ln n$. Because of the shape of
%this interval, it is natural to choose a beta prior with $a = b > 1$,
%which concentrates at $\frac{1}{2}$. Thus, for any $n$, $L$, and
%$c$, one can find a prior parameter $a$ that meets the second
%condition of Assumption~\ref{ass:hoelder-measure-observations}.
\end{proofof}

\begin{proofof}{Lemma~\ref{lem:normal}}
Since $N(x;\mu, \theta)$ is decreasing in $x^{2}$ and concave as a function of $\theta$. We have $\phi(t)\geq\inf_{\{x\mid||x||\leq B\},\theta\in[c_{1},c_{2}]}N(x;\mu, \theta)=\min\left\{\sqrt{\frac{c_{1}}{2\pi}}e^{\frac{-c_{1}c^{2}_{2}}{2}},\sqrt{\frac{c_{2}}{2\pi}}e^{\frac{-c^{3}_{2}}{2}}\right\}$. Then we can take
$$\constFamily=\min\left\{\sqrt{c_{2}/c_{1}}e^{\frac{c_{1}c^{2}_{2}}{2}}, e^{\frac{c^{3}_{2}}{2}}\right\}$$
For the normal distribution,~\eqref{eq:hoelder-measure-set} requires:
$2L\xdist{x}{y}  \sigma^2 \geq \left | 2\mu -x-y   \right| \left| x-y\right|$.
Taking the absolute log ratio of the Gaussian densities we have
\begin{align*}
& \frac{1}{2\sigma^2} \abs{\left((x - \mu)^2 - (y-\mu)^2\right)} \\
%=&
%\frac{1}{2\sigma^2} \abs{x^2 - y^2 - 2 \mu x + 2 \mu y }
%\\
%\leq &
%\frac{1}{2\sigma^2} \left(\abs{x^2 - y^2} + 2 |\mu| \abs{x-y}\right)
%\\
\leq&
\frac{\max\set{|\mu|, 1}}{2\sigma^2} \left(\abs{x^2 - y^2} + 2 \abs{x-y}\right).
\end{align*}
Consequently, we can set $\rho(x, y) = \abs{x^2 - y^2} + 2 \abs{x-y}$ and $L(\mu, \sigma) = \frac{\max\set{|\mu|, 1}}{2\sigma^2}$. Again, the trimmed exponential prior is given by $K\lambda e^{-\lambda \theta}$, where $K=(e^{-\lambda c_{1}}-e^{-\lambda c_{2}})^{-1}$. Thus the CDF at $L$ of this density is $K\left(e^{-\lambda c_{1}}-e^{-\lambda L}\right)$ for $L\in[\frac{c_{1}\max\set{|\mu|, 1}}{2},\frac{c_{2}\max\set{|\mu|, 1}}{2}]$ and $1$ for $L\geq \frac{c_{2}\max\set{|\mu|, 1}}{2}$. Thus the CDF at $L$ of this density is $K\left(e^{-\lambda c_{1}}-e^{
\frac{-2\lambda L }{\max\set{|\mu|, 1}}}\right)$. We choose $L_{0}$ to be $\frac{c_{1}\max\set{|\mu|, 1}}{2}$. Then we need to find $c$ such that
$$\xi(\Theta_{L})=\int^{L}_{c_{1}}K\lambda e^{-\lambda \theta}d\theta=K(e^{-\lambda c_{1}}-e^{-\lambda L})\geq 1- e^{-c\left(L-\frac{c_{1}\max\set{|\mu|, 1}}{2}\right)}.$$
By plugging $K$ into the inequality, we have
\begin{align*}
 e^{-c\left(L-\frac{c_{1}\max\set{|\mu|, 1}}{2}\right)}\geq \frac{e^{\frac{-2\lambda L}{\max\set{|\mu|, 1}}+\lambda c_{2}}-1}{e^{-\lambda(c_{1}-c_{2})}-1}.
\end{align*}
Since $e^{-\lambda\left(\frac{2\lambda L}{\max\set{|\mu|, 1}}-c_{2}\right)}\leq e^{-\lambda(c_{1}-c_{2})}$, it is sufficiency to find $c$ such that $$e^{-c\left(L-\frac{c_{1}\max\set{|\mu|, 1}}{2}\right)}\geq e^{-\lambda\left(\frac{2L}{\max\set{|\mu|, 1}}-c_{1}\right)}.$$
This is equivalent to have $c$ satisfying
$$c\left(L-\frac{c_{1}\max\set{|\mu|, 1}}{2}\right)\leq \lambda\left(\frac{2L}{\max\set{|\mu|, 1}}-c_{1}\right).$$
Then we can take $c=\frac{2\lambda}{\max\set{|\mu|, 1}}$ to satisfy the above inequality.
%OLD(Trim by projecting the mass to end points):
%Therefore with a trimmed exponential prior on its precision $\theta=1/\sigma^{2}$ we have $\xi(\theta_{L})=0$ for $L< \frac{c_{1}\max\set{|\mu|, 1}}{2}$, $\xi(\theta_{L})=1-\exp\left(\frac{-2\lambda L}{\max\set{|\mu|, 1}}\right)$ for $\frac{c_{1}\max\set{|\mu|, 1}}{2}\leq L\leq\frac{c_{1}\max\set{|\mu|, 1}}{2}$ and $\xi(\theta_{L})=1$ for $L\geq \frac{c_{2}\max\set{|\mu|, 1}}{2}$.
%Thus the assumption is satisfied with $c=\frac{2\lambda}{\max\set{|\mu|, 1}}$ and $L_{0}=\frac{c_{1}\max\set{|\mu|, 1}}{2}$.
\end{proofof}

\begin{proofof}{Lemma~\ref{lem:multinormal}}
Consider the likelihood log-ratio distance of multivariate normal distributions with precision matrix $A$:
$$\frac{1}{2}|x^{\top}A x-y^{\top}Ay|\enspace,$$
where $A$ is positive definite with eigenvalues $\lambda_{1}\geq \ldots \geq \lambda_{n}> 0$).
For simplicity, assume the mean to be a zero vector then
\begin{align*}
	|x^{\top}Ax-y^{\top}Ay|&=\left|\sum_{i,j}x_{i}x_{j}A_{i,j}-\sum_{i,j}y_{i}y_{j}A_{i,j}\right|\\
&=\left|\sum_{i,j}A_{i,j}(x_{i}x_{j}-y_{i}y_{j})\right| \\
&=|Tr(A(xx^{\top}-yy^{\top})')|\\
&\leq [Tr(A^{2})Tr((xx^{\top}-yy^{\top})(xx^{\top}-yy^{\top})')]^{\frac{1}{2}} \\
&=\left(\sum^{n}_{i=1}\lambda^{2}_{i}\right)^{\frac{1}{2}}||(xx^{\top}-yy^{\top})||_{F}\enspace.
\end{align*}

For mean equal to $\mu$, we have
$$\frac{1}{2}|(x^{\top}-\mu)A (x-\mu)-(y^{\top}-\mu)A(y-\mu)|\enspace.$$
%Let $x-\mu=z_{1}$ and $y-\mu=z_{2}$, b
By the above analysis we have the difference being bounded by
\begin{align*}
\frac{1}{2}\left(\sum^{n}_{i=1}\lambda^{2}_{i}\right)^{\frac{1}{2}}||(x-\mu)(x-\mu)'-(y-\mu)(y-\mu)')||_{F}\enspace.
\end{align*}
Note that
\begin{align*}
||(x-\mu)(x-\mu)'-(y-\mu)(y-\mu)')||_{F}=&||xx^{\top}-\mu(x^{\top}-y^{\top})-(x-y)\mu'-yy^{\top}||_{F}\\
\leq& ||xx^{\top}-yy^{\top}||_{F}+2||\mu(x-y)'||_{F}\\
=& ||xx^{\top}-yy^{\top}||_{F}+2||\mu||_{2}||(x-y)'||_{2}\\
\leq& \max\{1, ||\mu||_{2}\}(||xx^{\top}-yy^{\top}||_{F}+2||x-y||_{2})\enspace.
\end{align*}
\end{proofof}

\begin{proofof}{Lemma~\ref{lem:dbn}}
It is instructive to first examine the case where all variables are independent and we have a single draw from $P_\param$. Then
$P_\param(x) = \prod_{k=1}^K \param_{k, x_k}$ and
\begin{align}
\left|\ln \frac{P_\param(x)}{P_\param(y)}\right|
& =
\left|\ln \prod_{k=1}^K \frac{\param_{k, x_k}}{\param_{k, y_k}}\right|
 \leq
\sum_{k=1}^K
\left|\ln \frac{\param_{k, x_k}}{\param_{k, y_k}}\right|
\ind{x_k \neq y_k}
 \leq
\max_{i, j, k} \left|\ln \frac{\param_{k, i}}{\param_{k, j}}\right| \xdist{x}{y}\enspace.
\label{eq:bound-naive-bayes}
\end{align}
Consequently, if $\varepsilon \defn \min_{k,j} \param_{k.j}$ is the smallest probability assigned to any one sub-event, then $L > \ln 1 / \varepsilon$, since $\param_{k,j} \leq 1$.

In the general case, we have independent draws $x^t, y^t$, where $x^t \sim P_\param(x)$ and the variables $x^t_k$ have dependences defined through a graphical model, such that $P_\param(x) = \prod_k P_\param(x_k \mid x_{\Parents{k}})$, where $\Parents{k}$ are the parents of node $k$.
Similarly to \eqref{eq:bound-naive-bayes}, we write
\begin{align}
\left|\ln \frac{P_\param(x)}{P_\param(y)}\right|
& =
\left|\ln \prod_t \frac{P_\param(x^t)}{P_\param(y^t)}\right|
=
\left|\ln \prod_t \prod_k \frac{P_\param(x^t_k \mid x^t_{\Parents{k}})}{P_\param(y^t_k \mid y^t_{\Parents{k}})}\right| \nonumber
\\
& \leq
\sum_{t.k} \left|\ln \frac{P_\param(x^t_k \mid x^t_{\Parents{k}})}{P_\param(y^t_k \mid y^t_{\Parents{k}})}\right|
 \leq
\ln \frac{1}{\epsilon} \sum_{t.k} \ind{x_k^t \neq y_k^t \vee x_{\Parents{k}}^t \neq y_{\Parents{k}}^t}.
\label{eq:bound-dbn}
\end{align}
The last term is the number of times a value is different in $x$ and
$y$ times one plus the number of variables it affects. To model this,
let $v \in \Naturals^K$ be such that $v_k = 1 + \degree{k}$ and
define: $\xdist{x}{y} \defn v^\top \delta(x, y)$ and
$\delta_k(x,y) \defn \sum_{t} \ind{x_{k,t} \neq y_{k,t}}$.
Rewriting~\eqref{eq:bound-dbn} in terms of $\xdistChar$, we obtain $\left|\ln \frac{P_\param(x)}{P_\param(y)}\right| \leq \ln \frac{1}{\varepsilon} \cdot \xdist{x}{y}$ as desired.
\end{proofof}

%%% Local Variables:
%%% mode: latex
%%% TeX-master: "jmlr"
%%% End:

%\bibliographystyle{splncsnat}
\bibliography{references}

\end{document}